%% file: neurips_2023.tex
\documentclass{article}

\PassOptionsToPackage{numbers, compress, sort}{natbib}
\usepackage[preprint]{neurips_2023}
\usepackage{fancyhdr}

\usepackage[utf8]{inputenc} % allow utf-8 input
\usepackage{xcolor}         % colors
\definecolor{seablue}{RGB}{30, 144, 255}

\usepackage[T1]{fontenc}    % use 8-bit T1 fonts
\usepackage{hyperref}       % hyperlinks
\hypersetup{
    colorlinks=true,
    citecolor=seablue  
}
\usepackage{url}            % simple URL typesetting
\usepackage{booktabs}       % professional-quality tables
\usepackage{amsfonts}       % blackboard math symbols
\usepackage{nicefrac}       % compact symbols for 1/2, etc.
\usepackage{microtype}      % microtypography
\usepackage{subfig}      
\usepackage{appendix}      

\usepackage{algorithm}
\usepackage{algorithmic}

\usepackage{wrapfig}

\usepackage{amsmath}
\usepackage{amssymb}
\usepackage{mathtools}
\usepackage{amsthm}
\usepackage{bm}
\usepackage{bbm}

\usepackage[capitalize,noabbrev]{cleveref}
\AtBeginEnvironment{appendices}{\crefalias{section}{appendix}}

\usepackage{pdfpages}

\theoremstyle{plain}
\newtheorem{theorem}{Theorem}[section]

\newtheorem{lemma}[theorem]{Lemma}

\theoremstyle{definition}
\newtheorem{definition}[theorem]{Definition}

\theoremstyle{remark}

\newcommand{\m}[1]{\mathbf{#1}}
\newcommand{\nb}{\mathcal{B}}
\newcommand{\scm}{\mathcal{M}}
\newcommand{\SF}{\mathbb{SF}}

\usepackage[textsize=tiny]{todonotes}
\usepackage{csquotes}

\title{The Utility of ``Even if...'' Semifactual \\ Explanation to Optimise Positive Outcomes\thanks{This is a preprint, please cite the conference version at NeurIPS 2023.}
}

\author{%
  Eoin M.~Kenny\thanks{Contributed Equally.}\\
  Massachusetts Institute of Technology\\
  Cambridge, MA, U.S.A. \\
  \texttt{ekenny@mit.edu} \\
  \And
  Weipeng~Huang\footnotemark[2] \\
  Tencent Security Big Data Lab \\
  Shenzhen, Guangdong Province, China \\
  \texttt{fuzzyhuang@tencent.com} \\
}

\begin{document}

\maketitle

\begin{abstract}
When users receive either a positive or negative outcome from an automated system, Explainable AI (XAI) has almost exclusively focused on how to mutate negative outcomes into positive ones by crossing a decision boundary using counterfactuals (e.g., \textit{``If you earn 2k more, we will accept your loan application''}).
Here, we instead focus on \textit{positive} outcomes, and take the novel step of using XAI to optimise them (e.g., \textit{``Even if you wish to half your down-payment, we will still accept your loan application''}).
Explanations such as these that employ ``even if...'' reasoning, and do not cross a decision boundary, are known as semifactuals.
To instantiate semifactuals in this context, we introduce the concept of \textit{Gain} (i.e., how much a user stands to benefit from the explanation), and consider the first causal formalisation of semifactuals.
Tests on benchmark datasets show our algorithms are better at maximising gain compared to prior work, and that causality is important in the process.
Most importantly however, a user study supports our main hypothesis by showing people find semifactual explanations more useful than counterfactuals when they receive the positive outcome of a loan acceptance.

\end{abstract}

% \begin{figure}[!t]
%   \centering
%   \includegraphics[width=\linewidth]{imgs/title.pdf}
%   \caption{
%   semifactual Explanation \& Algorithmic Recourse: Datapoint $x$ has their loan accepted.
%   However, it is possible for them to intervene on several features to make their situation more favorable.
%   For instance, the $x'$ datapoints represents the user e.g. halving their down-payment, increasing the term over which they must repay the loan, or halving their savings, whilst still being accepted.
%   There are multiple ethically dubious reasons the bank may decide to withhold this information, but semifactual explanations make the user's options transparent.
%   }
%   \label{fig:title}
% \end{figure}

\section{Introduction}
\label{sec:intro}
\setcounter{footnote}{1} 
\footnotetext{Code available at: \url{https://github.com/EoinKenny/Semifactual_Recourse_Generation}}
Explainable AI (XAI) is broadly categorised into factual~\cite{bach2015pixel,kenny2023towards,ijcai2023p48} and contrastive explanations~\cite{miller2018contrastive,kenny2021generating}.
Within contrastive XAI, despite being neglected in comparison to counterfactuals, semifactuals are a major, fundamental part of human explanation, and have long been studied in psychology~\cite{boninger1994counterfactual, mccloy2002semifactual,sarasvathy2021even}, philosophy~\cite{bennett1982even,goodman1983fact,barker1991even}, and lately computer science~\cite{kenny2021generating,lu2022rationale,artelt2022even,aryal2023even,zhao2022generating,vats2022changes,artelt2023not}.
They take the form of \textit{\enquote{Even if $x$ happened, $y$ would still be the outcome}}.
Such reasoning has many potential uses as demonstrated by these prior works, but here we are focused on how semifactuals can help optimise positive outcomes for users, which (to the best of our knowledge) remains completely unexplored.

Our definition of counterfactuals is in line with Wachter et al.~\cite{wachter2017counterfactual}, where a test instance classified as $c$ must be mutated to cross a decision boundary into class $c'$.
Likewise, as established in the literature~\cite{kenny2021generating,artelt2022even}, we define a semifactual as an instance classified as $c$, which must be modified in such a way as to \textit{not} cross a decision boundary (and hence remain class $c$)~\cite{kenny2021generating}.
In recourse~\cite{ustun2019actionable,karimi2021algorithmic}, ``negative outcomes'' (e.g., a loan rejection) are generally mutated to produce ``positive outcomes'' (e.g., a loan acceptance) for users using counterfactuals. 
In our setting, we are assuming there was initially a positive outcome, and we are trying to mutate features to produce an even better situation for users, and in doing so \textit{not} cross the boundary into the negative outcome (i.e., using semifactuals).

Historically, counterfactuals have had obvious applications in computer science, such as explaining how to have a bank loan accepted rather than rejected, but applications for semifactuals as less clear. 
As such, the usage of semifactuals has often inadvertently defaulted to copying counterfactual research by also explaining negative outcomes (e.g., \textit{``Even if you double your savings, your loan will still be rejected''}~\cite{kenny2021generating,salimi2022addressing,aryal2023even}).
However, such an application for semifactual explanation perhaps has two main issues.
Firstly, it is debatable if these explanations convey useful information~\cite{artelt2022even}, whilst a counterfactual explaining how to cross a decision boundary and have a loan accepted has obvious utility~\cite{karimi2021survey}.
Secondly, such explanations make the user's situation seem helpless~\cite{mccloy2002semifactual}, in that they cannot possibly have their loan accepted, which raises ethical concerns~\cite{artelt2022even}.
However, our proposed framework can be used to not only overcome both of these issues, but actively \textit{contribute} to fairness.

Firstly, to try offer useful information for users, we flip the usual recourse problem and consider the user starting from a positive (rather than a negative) outcome.
In this setting, consider a user that has had their loan accepted, but might prefer to make a smaller down-payment on a loan application.
In this situation, our framework could present an explanation such as \textit{``Even if you half your down-payment, your loan will still be accepted''}, which seems to be more useful than explaining negative outcomes (see Section~\ref{Section:UserStudy}).
Secondly, because we are starting from a positive outcome, there is no danger of manipulating people into accepting a negative outcome, which guarantees fairness is this regard.
Now, with regards to optimising fairness even further, note that banks are not motivated to share such explanations even though they may help people, because (for example) larger down-payments are associated with lower risk on their behalf~\cite{benetton2018down}.
So, the usage of semifactuals in this application has clear potential to actively \textit{encourage} fairness and transparency.
As an aside, it is worth noting that although the focus of this paper is on financial applications, this research has broad impact on any domain for which the optimisation of a positive outcome is beneficial.
For instance, in medical applications, our framework could present explanations of the form \textit{``Even if you half your dose of drug $x$, you will still be at a low risk for disease $y$''}.
This is once again important for optimising fairness because people are frequently over-prescribed medicine with adverse side-effects~\cite{safer2019overprescribed}, but due to profit Big Pharma has no incentive to actively encourage this type of transparency.
Similar usage of semifactuals have also been proposed in smart agriculture to combat climate change~\cite{kenny2021generating}.

Our main contributions are: (1) the first explicit exploration of how to optimise positive outcomes with XAI, (2) the problem formulation for this which involved augmenting current semifactual research with the concept of \emph{Gain} (see Section~\ref{section:gain}), and (3) the premiere user test in the XAI literature for semifactuals, showing a clear application in which users find them more useful than counterfactuals.

\section{Literature Review}
When using contrastive explanation to explain loan acceptance decisions, to the best of our knowledge, this has only been explored by McGrath et al.~\cite{mc2018interpretable}.
Specifically, they suggest \textit{positive counterfactuals}, which show ``by how much'' a user had their loan accepted to help inform them when making future financial decisions.
While this is interesting information, we show that users find semifactual explanations more useful in loan acceptance situations than positive counterfactuals (see Section~\ref{Section:UserStudy}).

Semifactual explanation is growing in popularity~\cite{aryal2023even}, Kenny \& Keane~\cite{kenny2021generating} first explored the idea, but focused only on images.\footnote{Note there is other work on \textit{a-fortiori} explanations which have similar computational techniques to semifactuals~\cite{cummins2006kleor,doyle2004explanation,peters2023model}, they are justifications of the form \textit{``Because $x$ it true, $y$ must also be true''}.}
Artelt \& Hammer~\cite{artelt2022even} used diverse semifactuals to explain why an AI system refuses to make predictions due to having an unacceptably low certainty, but ignore how to explain either positive or negative outcomes.
Lu et al.~\cite{lu2022rationale} explain spurious patterns with semifactuals using a human-in-the-loop framework in NLP.
Zhao et al.~\cite{zhao2022generating} proposed a class-to-class variational encoder (C2C-VAR) with low computational cost that can generate semifactual images.
Vats et al.~\cite{vats2022changes} used generative models to produce semifactual image explanations for classifications of ulcers.
Lastly, for model exploration, Xie et al.~\cite{xie2023joint} sampled semifactual images with a joint Gaussian mixture model, and Dandl et al.~\cite{dandl2023interpretable} proposed deriving semifactual explanations from interpretable region descriptors.
In contrast to all these approaches, we are showcasing how semifactuals can be used to optimise positive outcomes for users (notably in causal settings).

From a user perspective, many have discussed the urgent need for comparative tests with semifactuals~\cite{salimi2022addressing,mertes2022alterfactual,kenny2021generating,warren2022features,mueller2021authoring}, with Aryal \& Keane~\cite{aryal2023even} pointing to the \textit{`paucity of user studies'} in the area.
However, the only such tests we are aware of are in the psychological literature over two decades ago~\cite{mccloy2002semifactual}.
Taking to this challenge, we conduct the first such test directly comparing semifactuals to counterfactuals in the XAI literature (see Section~\ref{Section:UserStudy}).

Our research is related to algorithmic recourse~\cite{ustun2019actionable} in that we are trying to ensure users are treated fairly by automated systems~\cite{karimi2021survey}.
In this area, Mothilal et al.~\cite{mothilal2020dice} explored counterfactual diversity, in that we should be offering users several explanations.
In addition, counterfactual robustness has been examined~\cite{dominguez2022adversarial}, which proposes that generated explanations should be robust to distributional shifts.
Lastly, causality has been argued as essential to providing plausible recourse~\cite{karimi2020algorithmic}.
We see these three facets as being important to our problem setting, and instantiate them in our framework.
There are other areas in recourse such as sequential decision making~\cite{naumann2021consequence,de2023synthesizing}, fairness~\cite{von2022fairness}, and privacy~\cite{pawelczyk2022privacy}, but we leave their exploration within semifactual explanation for future work.

As an aside, the literature on sufficiency could be conflated with semifactual explanation, as it describes a set of ``sufficient'' features for a prediction which, in the presence of the other features mutating, mostly doesn't affect the outcome~\cite{watson2021local,dhurandhar2018explanations,ribeiro2018anchors}.
However the techniques offer no insights for how to generate a meaningful semifactual.
More importantly though, if the sufficient features are the only actionable ones, then by definition we can't modify them to create a semifactual.

%%% user studies

% Lastly, as fairness is a major issue in algorithmic design~\cite{von2022fairness}, it is worth noting that semifactuals could be misused when explaining why loans are rejected~\cite{artelt2022even}.
% For example, they could say \textit{``Even if you earn double your salary your loan will still be rejected''. }
% This gives the users the impression that their situation is hopeless when it may not be~\cite{byrne2019counterfactuals,artelt2022even}.
% To avoid this, we propose the novel usage of semifactuals in the reverse situation when loans are accepted, where such misuse is not possible.

% Lastly, we note that recourse generation (and by extension our proposed semifactual augmentation) can often result in potential societal issues related to discrimination. However, our pre-defined actionability constraints help with this by e.g. omitting features such as a person's gender from being mutable (see Section 4). In addition, because we generate a diverse pool of explanations (see Section 3.1.4), these may also be filtered to help acquire fairer explanations.

% Vats et al.~\yrcite{newsf} also used semifactuals to explore the decision boundary in image domains, but this is not directly concerned with recourse.
% Lastly, Artelt \& Hammer~\cite{artelt2023even} explored how diverse semifactuals could help explain rejections, but the work is still preliminary, and

\begin{figure*}[!b]
  \centering
  \includegraphics[width=1\textwidth]{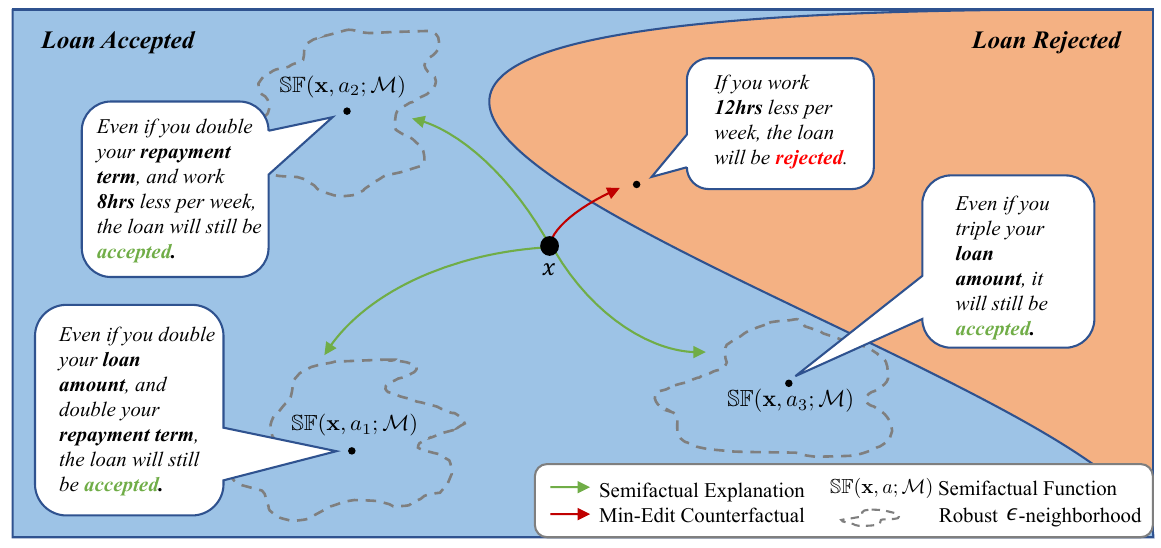}
  \caption{
  Semifactual Explanation to Optimise Positive Outcomes:
  An individual $\m{x}$ has their loan accepted, but there are several semifactual explanations which can help optimise their outcome.
  Our algorithm produces a set of semifactual explanations which \textit{maximise} the distance between $\m{x}$ and the final explanation $\mathbb{SF}(\m{x}, a; \scm)$.
  This allows the largest \textit{Gain} to be achieved so that the user gets the maximum benefit.
  In contrast, counterfactual algorithms are not suitable because they are designed to target the shortest path across a decision boundary.
  In addition, the semifactuals are robust to distributional shifts by constraining an $\epsilon$-neighborhood between them and the decision boundary.
  Note $\scm$ is the Structural Causal Model (SCM), see Section~\ref{section:framework}.
  }
  \label{fig:title}
\end{figure*}

\input{framework}

\input{method}

\section{Experiments \& Results}
\label{sec:expt}
Here we test S-GEN in both causal and non-causal settings.
We show the effectiveness of our method in optimising a user's positive outcome compared to baselines and open source our code (see~\Cref{sec:code}).
The actionability constraints are detailed in~\Cref{sec:actConsts}.
Baselines were modified to be appropriately compared, most importantly, we stopped counterfactual techniques before they crossed a decision boundary (thus generating semifactuals), and modified semifactual techniques to work on tabular data,~\Cref{sec:baselinesAppendix} details the peripheral modifications.

In the non-causal setting, we consider three datasets, Loan Application~\cite{lendingDataset}, German Credit~\cite{groemping2019south}, and BCSC~\cite{cancerData}.
All categorical variables are one hot encoded.
Three models were used, a decision tree, logistic regression, and na\"ive bayes, each with 30 random test data point explanation samples gotten by varying the random seed.
Note that because the range varied on each dataset, the results were normalised and averaged for each, but~\Cref{sec:full_figs} details each individual dataset for completeness.
For baselines, we modify three techniques, DiCE by Mothilal et al.~\cite{mothilal2020dice} (henceforth DiCE*), PIECE by Kenny \& Keane~\cite{kenny2021generating} (henceforth PIECE*), and Diverse semifactual Explanations of Reject by Artelt \& Hammer~\cite{artelt2022even} (henceforth DSER*).
Plausibility is measured as the distance between a generated semifactual(s) and the nearest training example; thus, the smaller the better.
Robustness is measured by MC sampling $n=100$ single feature perturbations of each semifactual $\bm{\theta}_i$, predicting their class, and returning a float between 0-1 of the success rate as described in Section~\ref{sec:fitfunction}.

In the causal setting, the Adult~\cite{KohaviAdult} and COMPAS~\cite{barenstein2019propublica} datasets are considered.
The SCMs from Nabi \& Shpitser~\cite{nabi2018fair} were used, and the structural equations from Dominguez et al.~\cite{dominguez2022adversarial}.
All categorical features are treated as real-valued.
We use the pre-trained MLP classifiers from Dominguez et al.~\cite{dominguez2022adversarial} and take 30 averaged samples from 5 random seeds.
As baselines we modify the technique of Karimi et al.~\cite{karimi2020algorithmic} [henceforth Karimi et al.(2021)*], and Dominguez et al.~\cite{dominguez2022adversarial} [henceforth Dominguez et al.(2022)*], the latter optimises with robustness in mind.
We optimise the relevant techniques to be robust in an $\epsilon=0.1$ hypersphere, and the C\&W adversarial attack by Carlini \& Wagner~\cite{carlini2017towards} measures robustness by checking if the nearest adversarial attack is outside this radius.

For all tests, the main metric of concern is gain, that is, the mean distance between a query and its generated semifactual(s), the larger this number, the better.
Diversity is also measured for all tests as the mean distance between all $m$ generated semifactuals for an individual $\m{x}$, the higher the number, the better.
To be in line with prior art, the $L_2$ norm is used in non-causal tests~\cite{artelt2022even}, and the $L_1$ in causal~\cite{dominguez2022adversarial}.
Note for causal tests the SCM guarantees plausibility so this metric is not reported.

\subsection{Non-Causal Results}

\begin{figure*}[!b]
  \begin{center}
  \includegraphics[width=1\textwidth]{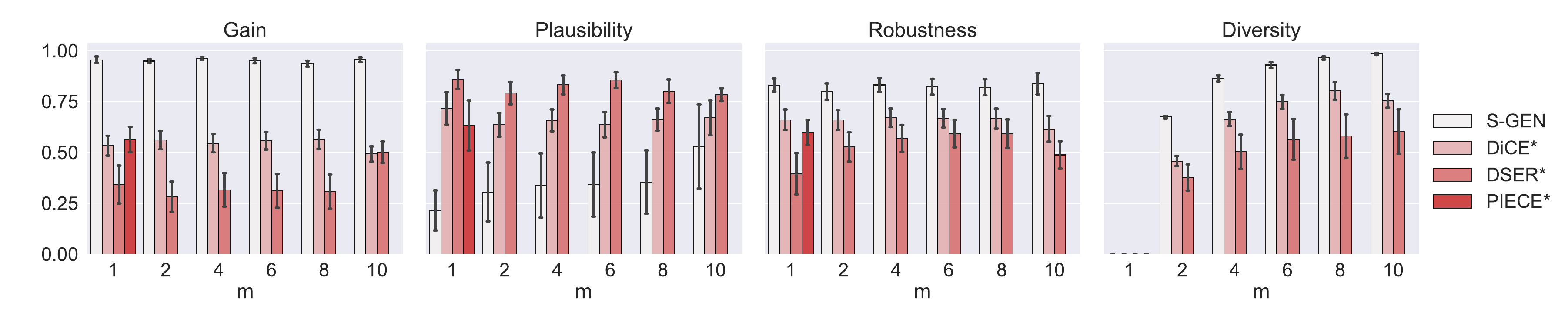}
  \caption{
  Results:
  The ability of S-GEN to create semifactuals is compared to DiCE*, DSER*, and PIECE*.
  Overall, S-GEN does the best, achieving significantly better results to all baselines in all tests.
  Note we normalised all results before averaging because each dataset has different scaling.
  % The most notable discrepancy to this occurs on the plausibility metric for the German Credit dataset, likely because this dataset was the smallest (1000 instances), and thus it is simply harder to find nearest neighbor matches for this metric, although S-GEN still performed best here at $m=1$.
  Standard error bars are shown.
  }
  \label{fig:nonCausalExpt}
  \end{center}
\end{figure*}

Our purpose here is to show that current methods are insufficient to meet the basic requirements for semifactual explanation discussed in Section~\ref{section:sfComponents}.
Specifically, a technique needs to optimise gain, while remaining plausible, robust, and offering diverse explanations.

Observing the average normalised results across all datasets (note robustness was not normalised since it is already 0-1 range), Figure~\ref{fig:nonCausalExpt} shows that S-GEN performed the best on all metrics for all values of $m$ (1-10).
The results demonstrate that traditional counterfactual approaches (DiCE*) are not suitable to achieve optimal gain, due to them focusing on minimising cost.
Moreover, methods built for semifactual generation specifically (i.e., DSER* \& PIECE*) that \textit{do} actually maximise gain somewhat, still fail to match the results of S-GEN.
This shows that S-GEN is superior to existing semifactual methods (and popular counterfactual approaches appropriately modified) for maximising a user's gain in positive outcomes.
Moreover, it does so while maintaining superior plausibility, robustness, and diversity in all tests.

\subsection{Causal Results}

% \begin{figure*}[!t]
%   \centering
%   \includegraphics[width=0.95\textwidth]{./imgs/exp_gain_user_study.pdf}
%   \caption{
%   Causal Experiment \& User Study Results:
%   (A/B) show the gain achieved by all methods both before and after considering the causal dependencies.
%   Firstly, note that S-GEN achieves significantly more gain than the alternatively proposed approaches.
%   Most importantly however, (A) shows there is significantly more gain achieved on the Adult data by S-GEN after taking causal dependencies into account, showing the importance of a causal formalization.
%   (C) Shows the user study results, where people perceive semifactual explanation as being significantly more useful than counterfactuals in the positive outcome of having a loan accepted.
%   Standard error bars are shown.
%   }
%   \label{fig:expt2_3}
% \end{figure*}

\begin{figure}[!t]
  \begin{center}
  \subfloat[Causal Experiment: Adult]{
    \includegraphics[width=0.31\textwidth]{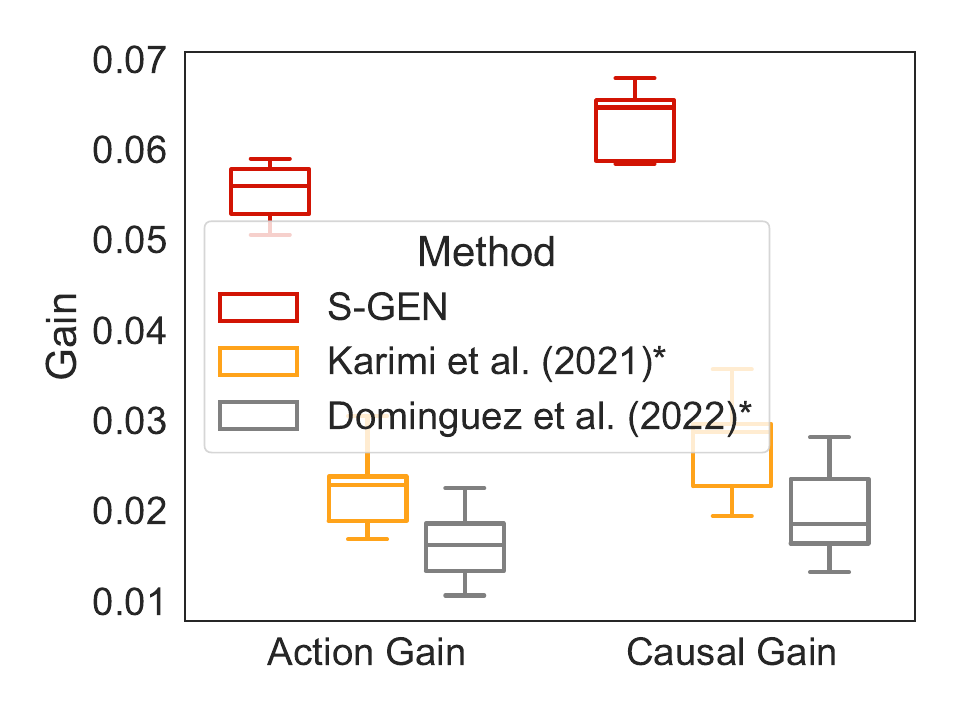} \label{fig:adult}
  }~~
  \subfloat[Causal Experiment: COMPAS]{
    \includegraphics[width=0.31\textwidth]{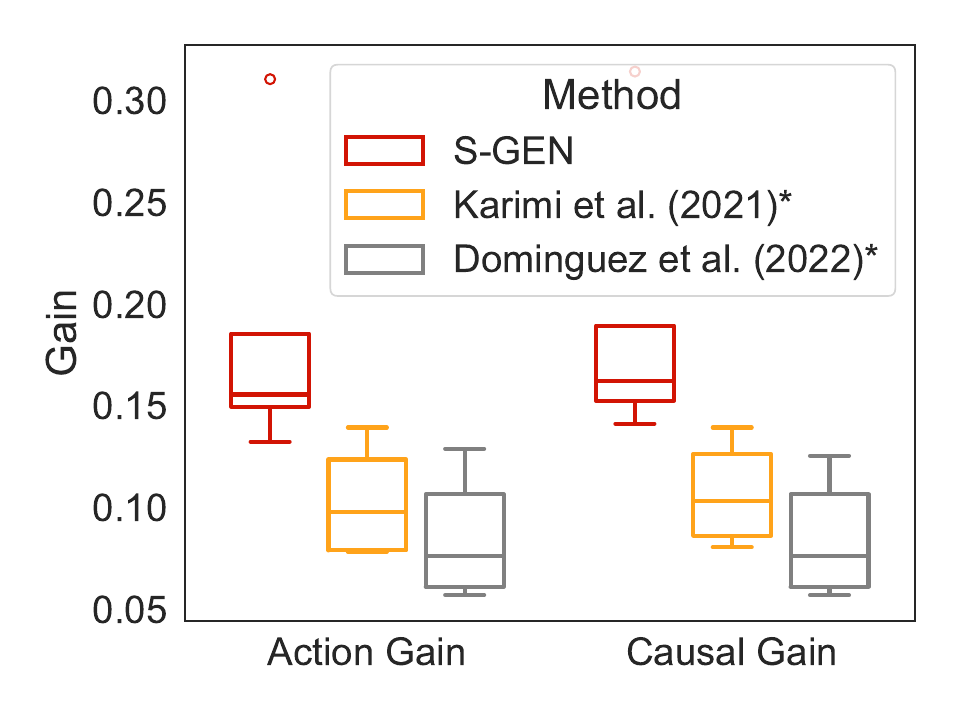} \label{fig:compas}
  }~~
  \subfloat[User Study]{
    \includegraphics[width=0.31\textwidth]{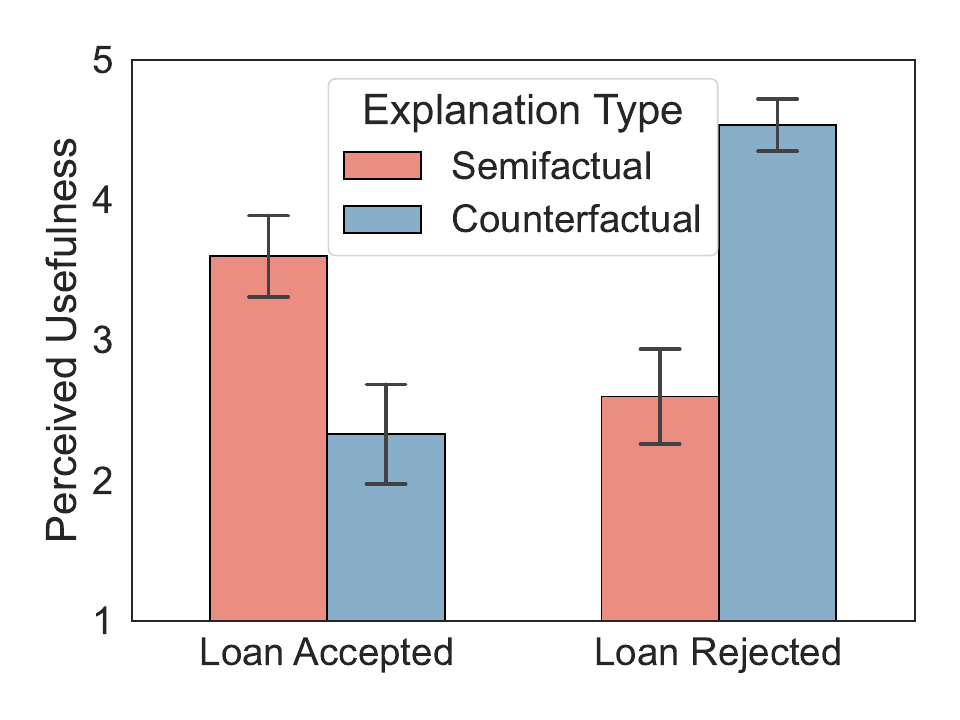} \label{fig:user_study}
  }
  \caption{
  Causal Experiment \& User Study Results:
  (a/b) show the gain achieved by all methods both before and after considering the causal dependencies.
  Firstly, note that S-GEN achieves significantly more gain than the alternatively proposed approaches.
  Most importantly however, (a) shows there is significantly more gain achieved on the Adult data by S-GEN after taking causal dependencies into account, showing the importance of a causal formalisation.
  (c) Shows the user study results, where people perceive semifactual explanation as being significantly more useful than counterfactuals in the positive outcome of having a loan accepted.
  Standard error bars are shown.
  }
  \label{fig:expt2_3}
  \end{center}
\end{figure}

We evaluate our algorithm in a causal setting where the SCMs and structural equations are known.
The primary purpose of this test is to demonstrate that the semantics of semifactual ``even if'' thinking is better captured in a causal setting due to dependencies being taken into account when calculating a person's gain.
With regard to diversity, we fix $m$ to the maximum number of feature sets available from the actionable features (so only one $m$ value is tested).
% Hence, to avoid any bias, we use the pre-defined SCMs by xxx and the structural equations of xxx to evaluate this.
% Specifically, we wish to observe if an individual's gain will be significantly affected by considering causal relations between features; if it is, then we may assert that the calculation of a person's gain significantly benefits from causal formulation.

\Cref{fig:adult,fig:compas} show the initial gain achieved by a person after taking a certain action (i.e., the \textit{Action Gain}), and how this gain transforms after considering the causal relationship between features (i.e., the \textit{Causal Gain}).
% Figure~\ref{fig:expt2_3}(A/B) shows the initial gain achieved by a person after taking a certain action (i.e., the \textit{Action Gain}), and how this gain transforms after considering the causal relationship between features (i.e., the \textit{Causal Gain}).
Firstly, the total gain achieved by S-GEN is much larger than the baselines in both datasets and hence consistent with our non-causal tests.
More importantly however, the change in gain a person achieves after considering the causal relations in the adult dataset is significantly higher both in significance testing and effect size ($0.055 \pm 0.001$ v. $0.063 \pm 0.001$; t-test $p<0.02$; Cohen's $d=2.24$), showing it is beneficial to consider causality when calculating a person's gain.
The results of diversity put S-GEN first also ($\mbox{S-GEN} = 0.84 \pm 0.09$ v. $\mbox{Karimi}=0.43 \pm 0.03$ v. $\mbox{Domineguez}=0.34 \pm 0.03$).
In robustness, both S-GEN and Domineguz et al. (2022)* did reasonably well ($\mbox{S-GEN}=87\%$ success v. $\mbox{Domineguz}=54\% $ success), but Karimi et al. (2021)* did not ($7.2\%$ success), likely due to the latter not being designed for this.

% \begin{wrapfigure}{R}{0.5\textwidth}
%     \centering
%     \includegraphics[width=0.47\textwidth]{./imgs/causal_expt.pdf}
%   \caption{
%   Causal Results:
%   Illustrated is the gain of taking the ``recourse'' action (equivalent to the notion of cost~\cite{ustun2019actionable}), and the actual gain achieved after considering causality.
%   Firstly, S-GEN achieves a much higher gain than other approaches overall.
%   Secondly, and most importantly, there are significant differences between two calculations, showing the importance of formalizing \emph{gain} as we have and considering causality in semifactual explanation (note this also validates Theorem~\ref{sec:gaintheory}.
%   }
%   \label{fig:diverse}
% \end{wrapfigure}

% \label{table:diverse}

% \end{table}

\section{User Evaluation}
\label{Section:UserStudy}
The primary motivation behind this work is the hypothesis that semifactual explanation would be preferred by users over counterfactuals in positive outcome settings.
To test this assumption, we design the first user test in XAI directly comparing the two.
Specifically, we show users three materials in which a person has a bank loan accepted, and three in which they don't.
Users were then shown both explanation types for each material, and asked to rate on a scale from 1-5 how useful each were. 
So, the study was a within-subjects design, and the condition was the explanation type.
Note that although we are studying the effect of the explanation type on loan acceptance, the loan rejection scenarios were also included to balance people's view of the problem setting, and as attention checks to verify that users were engaging with the materials and varying their scores accordingly.
For analysis, each user's scores for counterfactuals and semifactuals were averaged in both loan acceptance and rejection materials into four decimal scores per user, thus allowing us to analyse the discrete Likert scores with t-tests~\cite{kenny2021explaining}.
As is a popular approach~\cite{hoffman2018metrics}, we don't explicitly define what ``useful'' means to users, but rather let them use their own natural interpretation, as the results returned were reasonably consistent across individuals, they appear to have converged on an common interpretation of this word.
The null hypothesis is that people will find both explanation types not significantly different in loan acceptance.
The alternative is that people will find semifactuals significantly more  useful in loan acceptance.

A power analysis~\cite{cohen1992statistical} of two dependent means with an effect size $dz=0.8$, alpha $\alpha=0.05$, and power ($1-\beta$ err prob)=0.9 informed a sample of 15 was appropriate for t-tests.
Users were gathered from Prolific.com, 8 males, 7 females, aged 18+, native English speakers, and from the U.S.
People were paid \$12/hr, which totalled \$35.
The semifactuals were generated with S-GEN, and the counterfactuals with DiCE~\cite{mothilal2020dice}, notably these are equivalent to \textit{positive counterfactuals} by McGrath et al.~\cite{mc2018interpretable} for explaining loan acceptance situations.
The study obtained IRB approval from MIT.

All users engaged and changed their ratings significantly depending on whether a loan was accepted or rejected, so none were excluded.
\cref{fig:user_study} shows users find semifactuals significantly more useful in loan acceptance (S-GEN=3.60$\pm$0.27 v. DiCE=2.33$\pm$0.34; $p$ $<$ .005) compared to rejections when counterfactuals are preferred (S-GEN=2.6$\pm$0.32 v. DiCE=4.53$\pm$0.17; $p$ $<$ .0001).
Hence we reject the null and lend credible evidence that semifactuals are more useful to explain positive outcomes.

% \begin{figure}[!t]
%   \centering
%   \includegraphics[width=0.5\textwidth]{./imgs/user.pdf}
%   \caption{
%   Results:
%   The ability of S-GEN to create semifactuals is compared to DiCE* and PIECE*.
%   Overall, S-GEN does the best, achieving significantly better results to both baselines on 11/16 tests.
%   Moreover, S-GEN was only significantly worse than either baseline on a single test (i.e., plausibility on German Credit), with the remaining four tests being competitive between methods.
%   % The most notable discrepancy to this occurs on the plausibility metric for the German Credit dataset, likely because this dataset was the smallest (1000 instances), and thus it is simply harder to find nearest neighbor matches for this metric, although S-GEN still performed best here at $m=1$.
%   Standard error bars are shown.
%   }
%   \label{fig:diverse}
% \end{figure}

\section{Discussion}
\label{Section:discussion}
Although much XAI work has explored how to explain positive outcomes, to the best of our knowledge, no consideration has been given towards explaining how to \textit{optimise} them.
Here, we have taken the novel step of exploring this, and showed how semifactuals are especially suited for the purpose.
This required building on prior work in semifactuals by (1) introducing the concept of \emph{Gain}, (2) re-framing them in a causal setting, and (3) conducting their first user test in XAI.
Perhaps the notable limitation of our work is that although we have shown people do perceive semifactuals as being more useful in positive outcomes, we have not demonstrated this quantitatively, notably because of the difficulties acquiring an appropriate user base alongside the ethical considerations of such a study.
Moreover, considering a casual formulation of semifactuals requires an SCM, which is not always realistic, but we have provided a non-causal algorithm for these situations.
In future work, it would be interesting to formalise the utility of semifactuals for optimising positive outcomes in other domains such as robotics, which likely requires other considerations.

\section*{Acknowledgements}
The authors would like to thank Neil J. Hurley, alongside Mark T. Keane and Ruth M.J. Byrne who both inspired early consideration of the ideas in this paper. 
The authors would also like to thank MIT for their support in this project. 
This research wasn't directly supported by any grants or funding. 
We hope readers find the ideas interesting and useful for application. 

% \textbf{Do not} include acknowledgements in the initial version of
% the paper submitted for blind review.

% If a paper is accepted, the final camera-ready version can (and
% probably should) include acknowledgements. In this case, please
% place such acknowledgements in an unnumbered section at the
% end of the paper. Typically, this will include thanks to reviewers
% who gave useful comments, to colleagues who contributed to the ideas,
% and to funding agencies and corporate sponsors that provided financial
% support.

% % In the unusual situation where you want a paper to appear in the
% % references without citing it in the main text, use \nocite
% \nocite{langley00}

\bibliographystyle{abbrv}
\bibliography{example_paper}

\begin{appendices}
\input{appendix}
\end{appendices}

\end{document}

%% file: framework.tex
\section{Semifactual Framework}
\label{section:framework}
In this section, we describe the basic definitions and assumptions for our semifactual framework to optimise positive outcomes for users, before formalising it under the concept of \emph{Gain} (i.e., how much a user stands to benefit from the explanation) in a causal setting, neither of which has been considered before.
As an aside, we also show how the established concepts of plausibility, robustness, and diversity can be made fit into the objective to offer better explanations.
Finally, we reflect on the theoretical properties of the framework.
% , and because we wish to utilize them as additional evaluation metrics in experiments later.
% Learning from prior work on recourse~\citep{dominguez2022adversarial,ustun2019actionable}, we formalize our objective function under the headings of robustness, plausibility, and diversity.

\subsection{Definitions}
Let us denote an individual $\m{x} \in \mathcal{X}$ with $k$ \emph{mutable} features $\m{X} = \{X_1, \dots, X_k\}$.
% For any $j$ in $1, \dots, k$, we define $x_j \in \mathcal{X}_j$ where $\mathcal{X}_j \subseteq \mathbb{R}$.
% Let $\mathcal{X} \coloneqq \mathcal{X}_1 \times \dots \times \mathcal{X}_k \subseteq \mathbb{R}^k$, and thus $\m{x} \in \mathcal{X}$.
Given the individual $\m{x}$, a set of actions can be applied to $\m{x}$ where each action $a(\m{x})$ is also a $k$-dim vector.
As in prior work~\cite{karimi2021algorithmic}, we apply $a(\m{x})$ and $a$ exchangeably, since the individual $\m{x}$ will always be fixed.
We explicitly exclude features that are either \emph{immutable} or \emph{non-actionable}.
% For a feature $x_j$, we define $A(x_j)$ to return the set of \emph{actionable} feature values, such that $A(x_j) \in \mathcal{A}_j \subseteq \mathbb{R}$.
% Similarly, we write $\mathcal{A} \coloneqq \mathcal{A}_1\times \dots \times \mathcal{A}_k \subseteq \mathbb{R}^k$.
% Therefore, we define a set of human-constrained actionable ranges $\mathcal{A}(\m{x})$ as
% $\mathcal{A}(\m{x}) = \{ a_j \in A(x_j): j = 1, \dots, k \}  \subseteq \mathcal{A}$.
Adopting Pearl's $do()$ operator~\cite{pearl2009causality}, an action can be defined as $a(\m{x})=do(\m{X} \coloneqq \bm{\theta})$, or simply $do(\bm{\theta})$, to force a hard intervention of replacing $\m{x}$ by $\bm{\theta}$ where $\bm{\theta} \in \mathcal{X}$.
It implies that, for each feature, $X_i \coloneqq \theta_i$ for the individual $\m{x}$.
If the action $do(\bm{\theta})$ imposes no change, $\m{x} = \bm{\theta}$ holds.
We further denote a set of human-constrained actionable ranges $\mathcal{A} = \{ a(\m{x}) = do(\bm{\theta}): \forall \bm{\theta} \in \mathcal{X} \}$.
% We write $a_j \simeq 0$ to represent that action $a$ has no change on the $j$th feature and vice versa, which applies to both cases (1) the action leaves out the feature, and (2) the feature is non-actionable.
Note that the actions have to be mutable and explicitly exclude any action which keeps the individual in the same position.

The non-causal semifactual interaction between $\m{x}$ and $a(\m{x})$ is defined by $\SF: \mathcal{X} \times \mathcal{A} \mapsto \mathcal{X}$.
That is, the individual $\m{x}$ taking action $a(\m{x})$ will lead to another representation  $\bm{\theta} \in \mathcal{X}$ representing that person's recourse.
% We consider a specific function $S(\cdot, \cdot)$ rather than a plus operator ``$+$'' here in order to cover the categorical features as well as the numeric ones.
Now, a structural semifactual is defined which considers the dependence between the related features~\citep{dominguez2022adversarial,karimi2021survey}.
We denote the structural causal model (SCM) by $\scm = (\m{S}, \mathbb{P}_{U})$ where $\m{S}$ are a set of structural equations and $\mathbb{P}_U$ is the distribution over the exogenous variables $U \in \mathcal{U}$.
Consider that in a causal graph, there is a set of causal parents for each feature $x_i \in \m{x}$, denoted by $\mathrm{Pa}_i$.
We denote the structural equations as $\mathbf{S} = \{x_i \coloneqq g_i(\mathrm{Pa}_i, U_i): i=1,\dots,k\}$ where $g_i(\cdot)$ is a deterministic function that describes the causal relationship for $x_i$, and depends on the exogenous variable $U_i \in \mathcal{U}$ alongside the corresponding parent set $\mathrm{Pa}_i$.
Hence, $\mathbf{S}$ induces a mapping $\mathbb{S}: \mathcal{U} \mapsto \mathcal{X}^{*}$ and its inverse mapping $\mathbb{S}^{-1}: \mathcal{X} \mapsto \mathcal{U}$.
% , where
% \begin{align*}
%   \mathbb{S}(\mathbb{S}^{-1}(\m{x})) = \m{x} \quad \forall \m{x} \in \mathcal{X} \,.
% \end{align*}
Let $f \circ g(x) = f(g(x))$ which can be extended to more functions.
Hence, we specify the SCM-processed semifactual by $\SF(\m{x}, do(\bm{\theta}); \scm)$ to denote the transition between the states by taking a certain action through an SCM $\scm$, where
\begin{equation}
\bm{\theta}' = \SF(\m{x}, do(\bm{\theta}); \scm) \coloneqq \mathbb{S} \circ \mathbb{S}^{-1} \circ  \SF(\m{x}, do(\bm{\theta}); \scm)\, .
\end{equation}
% To simplify the notation, we define a model $S_{\scm}: \mathcal{X} \times \mathcal{A} \mapsto \mathcal{X}$ that defines the transition between the states by taking a certain action through an SCM $\scm$, where
% $\bm{\theta}' = \SF(\m{x}, a; \scm) \coloneqq \mathbb{S} \circ \mathbb{S}^{-1} \circ  \SF(\m{x}, a)$.
If all features are \emph{independently manipulable}, we have $\bm{\theta}' = \bm{\theta} = \SF(\m{x}, do(\bm{\theta}); \scm) = \SF(\m{x}, do(\bm{\theta}))$.
Therefore, $\SF(\m{x}, do(\bm{\theta}); \scm)$ is a more generalised formulation which covers the non-causal case.
Lastly, we assume a binary model that generates the score for the users is $h$, where $h: \mathcal{X} \mapsto \{0, 1\}$ by which we can simply consider that $1$ means a positive outcome (e.g., a loan acceptance) and $0$ is a negative outcome (e.g., a loan rejection).
We set a lower threshold $\psi$ that separates the decision boundary.
For the form defined above, $\psi=0.5$ is a reasonable threshold that fits all situations well.

% \begin{equation}
% \bm{\theta}' = \SF(\m{x}, a; \scm) \coloneqq \mathbb{S} \circ \mathbb{S}^{-1} \circ  \SF(\m{x}, do(\bm{\theta}); \scm)\, .
% \end{equation}

% \begin{equation}
% \bm{\theta}' = \SF(\m{x}, a_1; \scm) \coloneqq \mathbb{S} \circ \mathbb{S}^{-1} \circ  \SF(\m{x}, do(\bm{\theta}); \scm)\, .
% \end{equation}

% \begin{equation}
% \bm{\theta}' = \SF(\m{x}, a_2; \scm) \coloneqq \mathbb{S} \circ \mathbb{S}^{-1} \circ  \SF(\m{x}, do(\bm{\theta}); \scm)\, .
% \end{equation}

% \begin{equation}
% \bm{\theta}' = \SF(\m{x}, a_3; \scm) \coloneqq \mathbb{S} \circ \mathbb{S}^{-1} \circ  \SF(\m{x}, do(\bm{\theta}); \scm)\, .
% \end{equation}

% \subsection{Framework Formulation}
% As eluded to earlier, we now respectively elaborate on the concepts of gain, plausibility, robustness, and diversity.
% \cref{tb:components} summarizes the components and their roles.
% Finally, we then combine these components into one objective.

% \begin{table}[!ht]
% \caption{Components of the semifactual framework. The discussion of the notation will be deferred to the latter sections.}
% \label{tb:components}
% \begin{center}
% \begin{tabular}{c c c}
% \toprule
% \textbf{Component} & \textbf{Role} & \textbf{Notation}\\
% \midrule
% gain & objective & $G$ \\
% robustness & constraint & $H_{p}, H_a$ \\
% plausibility & objective & $W$ \\
% diversity & objective (regularization) & $R$ \\
% \bottomrule
% \end{tabular}
% \end{center}
% \end{table}

\subsection{Framework}
We define our semifactual framework as one centering on gain ($G$) that is weighted by plausibility ($P$), regularization in the form of diversity ($R$), and hard constraints in the form of robustness ($H$), indexed by $j$.
All of the components are parameterized with $\m{x}$ and a subset of suggestions $\{a_1, \dots, a_m\}$.
Letting $f(\cdot)$ be a function composed by gain and some weighting (i.e., plausibility for us), the causal semifactual framework is defined as
\begin{align}
\label{eq:sf_framework}
\max_{a_1, \dots, a_m} &  \quad \frac 1 m \sum_{i=1}^m f(G(\m{x}, a_i), P(\m{x}, a_i)) + \gamma R(\{\bm{\theta}_1, \dots, \bm{\theta}_m \}) \nonumber \\
\text{s.t.} &  \quad \bm{\theta}_i = \SF(\m{x}, a_i; {\scm}), H_j(\bm{\theta}_i) \ge 0, \forall i, j
\end{align}
where the regularisation and hard constraints can be multiple and indexed with $i$ and $j$, respectively.
One may define a similar formulation for the non-causal case (see Section~\ref{sec:nonCausalImplementation}).
% \end{definition}
% In the paper, we implement a \textit{Gain} function, with plausibility as weight, diversity as regularization, and robustness as hard constraints.
We defer all details of the components until~\cref{section:sfComponents}.

% \begin{align*}
% \forall k=1,\dots,m: \quad a_k \in \arg\max_{a} \left[ f(P, G) + \sum_i \gamma_i R_i + \min_{\lambda_j} \sum_j \lambda_j (H_j - \psi_j) \right]
% \end{align*}

\subsection{Optimising Positive Outcomes with \emph{Gain}}
\label{section:gain}
For the core of the objective we appeal to the notion of gain.
Note that \textit{gain} is similar to the idea of \emph{cost} commonly used in recourse~\cite{ustun2019actionable}, but there are three crucial differences.
First,  we are trying to \textit{maximise} gain, rather than \textit{minimise} cost~\cite{karimi2021survey}.
Second, gain ideally considers the causal dependencies between features in its function, whilst cost typically only considers the user's action(s)~\cite{karimi2021survey}.
Third, gain is further subdivided into positive and negative polarities.
% To elaborate on this last point, prior work on semifactuals has maximized $L_2$ distance between a test instance and explanatory one to define good semifactuals~\cite{kenny2021generating,artelt2022even}, irregardless of what features are mutated, but gain is designed to be more meaningful in application by purposefully trying to change features that would benefit a user.
% Hence, we define \textit{positive gain} as modifying features in such a way which favorably mutates an individual user's current situation.
To elaborate on this last point, take the example of a user who has their loan application for buying a new house accepted.
In this situation, if they desired to spend more time away from work with family, they would experience \textit{positive gain} if they could work less hours per week and still have their loan accepted (see~\cref{fig:title}).
Conversely, if this person increased the number of hours they worked, they would experience \textit{negative gain}.
Notably, positive/negative gain is not necessarily connected to the model's probabilities (see $a_1$ in~\cref{fig:title} moving away from the decision boundary).
Similar to actionability constraints which can offer individualised recourse~\cite{ustun2019actionable}, what is positive/negative gain must be manually defined for each individual.
As prior work on semifactuals simply maximised the $L_2$ distance between a test instance and explanatory one to define good explanations~\cite{kenny2021generating,artelt2022even}, we introduced the concept of gain to make them more meaningful in application.

More formally, we define the gain function by $G: \mathcal{X} \times \mathcal{A} \to \mathbb{R}$.
By denoting $\bm{\theta} = \SF(\m{x}, a; \scm)$, we decompose the function as follows:
\begin{align}
\label{eq:gain}
G(\m{x}, a)
&\coloneqq \mathcal{P}_{SF} \circ \delta(\m{x}, \bm{\theta}) = \mathcal{P}_{SF} \circ \delta(\m{x}, \SF(\m{x}, a; \scm)) \,
\end{align}
where $\mathcal{P}(\cdot, \cdot)$ is an oracle function that computes the payoff  based on the vectorised difference between $\m{x}$ and $\bm{\theta}$, i.e., $\delta: \mathcal{X} \times \mathcal{X} \mapsto \mathbb{R}^k$ which is a symmetrical difference function between the two feature representations.
The subscript of $\mathcal{P}_{SF}$ denotes a semifactual.
In interpretation, the gain function compares two states, (1) the original feature vector $\m{x}$, and (2) the SCM-processed end state $\bm{\theta}$ which was led to through $\m{x}$ taking action $a$.

\paragraph{Why is \textit{Gain} not necessarily equivalent to \textit{Cost}?}
% Informally, aside from the opposing min max objectives, cost typically only considers actions, whilst gain also considers causal dependencies.
% Moreover, gain considers what is \textit{positive gain} for each user with each feature mutation in the calculation (and conversely, \textit{negative gain} also).
% As these latter considerations must be manually defined by the user's individual preferences (e.g., some users will considering increasing the amount of hour they work positive/negative gain based on their preferences), this means gain will occasionally be different to cost.
Formally, to enable the comparison, we write the cost function (denoted by $C(\m{x}, a)$) as
\begin{align}
\label{eq:cost}
C(\m{x}, a)
&=
- \mathcal{P}_{CF} \circ \delta(\m{x}, \SF(\m{x}, a)) \,
\end{align}
which builds on the fact that cost solely considers the feature change.
Note that $\SF$ is equivalent to the notion $\mathbb{CF}$ in~\cite{karimi2021algorithmic}, what makes our approach different is the consideration of positive outcomes and gain.
Our finding is that gain in semifactuals (SFs) is not necessarily equivalent to cost in counterfactuals (CFs) where the equivalence ignores the sign of both quantities, as formally stated as follows.
\begin{theorem}
\label{thm:gaintheory}
Even if $\mathcal{P}_{SF}(\cdot, \cdot) \equiv \mathcal{P}_{CF}(\cdot, \cdot)$, gain and cost are not necessarily equivalent ignoring the sign.
\end{theorem}
\begin{proof}
Note that SCMs are also considered in counterfactual recourse~\citep{karimi2020algorithmic,karimi2021survey,dominguez2022adversarial}.
However, in this prior research SCMs are typically applied for enforcing hard plausibility constraints, not in the computation of a user's cost.
In contrast, our gain function takes the SCM-processed semifactual $\bm{\theta}'$ as an input.
% It implies that the cost function only cares about the features with changes, denoted by $\m{f}_{CF}$, i.e.,
% \begin{align*}
% \m{f}_{CF}
% &= \{ j: \delta(\m{x}, \SF(\m{x}, a))_j \ne 0, j = 1, \dots, k \}
% \equiv \{ j: a_j \ne 0, j = 1, \dots, k \} \,
% \end{align*}
% where $a_j = 0$ indicates that the $j$th feature is not changed.
% The cost function is then a function that considers $a_{CF}$ and thus impact the feature subset $\m{f}_{CF}$ only.
% However, gain function is designed to produce a score by comparing two states as described earlier.
% Therefore, the features which will influence the gain are
% \begin{align}
% \m{f}_{SF} = \{j : \delta(\m{x}, \SF(\m{x}, a; \scm))_j \ne 0, j = 1, \dots, k\} \,.
% \end{align}
% We apply the contradiction techniques here.
We employ proof by contradiction here.
Assume that cost and gain are equivalent ignoring sign so that, without loss of generality,
\begin{align}
\label{eq:contradiction_assumption}
|G(\m{x}, a)| = |C(\m{x}, a)|
& \iff |\mathcal{P} \circ \delta(\m{x}, \SF(\m{x}, a; \scm))| = |\mathcal{P} \circ \delta(\m{x}, \SF(\m{x}, a))| \nonumber \\
& \iff \delta(\m{x}, \SF(\m{x}, a; \scm)) = \delta(\m{x}, \SF(\m{x}, a))
%  % \nonumber \\
% & \iff \delta(\m{x}, \SF(\m{x}, a; \scm))_j = \delta(\m{x}, \SF(\m{x}, a))_j \quad \forall j
\end{align}
holds.
However, SCMs can result in possibly more features being changed since some features could be others' causal parents and those causal children will change their values accordingly.
By denoting $\bm{\theta} = \SF(\m{x}, a)$ and $\bm{\theta}' = \SF(\m{x}, a')$, we consider the general case as follows: 
\begin{multline}
  | \delta(\m{x}, \bm{\theta}) | - | \delta(\m{x}, \bm{\theta}') |
= \sum_i | \delta(\m{x}, \bm{\theta})_i | - \sum_i | \delta(\m{x}, \bm{\theta}')_i | \\
=\sum_{\{i: \theta_i = \theta_i' \} \cup \{i: \theta_i \ne \theta_i' \}} |\delta(\m{x}, \bm{\theta})_i| - |\delta(\m{x}, \bm{\theta}')_i|
= 0 + \sum_{\{i: \theta_i \ne \theta_i' \}} |\delta(\m{x}, \bm{\theta})_i| - |\delta(\m{x}, \bm{\theta}')_i|
% &=\sum_{i \in \{i': a_{i'} \simeq 0\}} 0 - |\delta(\m{x}, \bm{\theta}')_i|
\le 0 \label{eq:contradiction_result}
\,,
\end{multline}
which contradicts with~\Cref{eq:contradiction_assumption}.
Thus, even if the oracle function for calculating the payoff is the same, gain and cost are still not necessarily equivalent.
Also, the equality in~\Cref{eq:contradiction_result} holds when all features are independently manipulable or the changed features are independently manipulable of the remaining features, so that $\SF(\m{x}, a)=\SF(\m{x}, a; \scm)$.
The proof completes here.
\end{proof}

% Assume that a binary model that generates the score for the users is $h$ where $h: \mathcal{X} \mapsto \{0, 1\}$ by which we can simply assume that $1$ means the loan acceptance and $0$ is the opposite.
% We set a lower threshold $\psi$ that separates the decision boundary.
% For both forms defined above, $\psi=0.5$ can be a reasonable threshold that fit all situations well.

% On the other hand, we can maximize $ G$ instead of $G$, which is more convenient for optimization yet preserves its monotonic property.
% The lightest version of the semifactual objective is
% \begin{equation}
% \label{eq:sf-obj}
% \begin{aligned}
% \max \quad & G(\m{x}, a) \\
% \mbox{subject to} \quad & a \in \mathcal{A}^+(\m{x}),~ h(\SF(\m{x}, a; \scm)) > \psi \, .
% \end{aligned}
% \end{equation}

\subsection{Semifactual Components}
\label{section:sfComponents}
Here, we detail how to incorporate the concepts of plausibility, robustness, and diversity into our framework for maximising gain, because they are agreed upon as important in the literature and useful for evaluation.
While plausibility and diversity have been explored in semifactual explanation~\cite{kenny2021generating,artelt2022even}, robustness and causality (and indeed an objective balancing all together) have not, yet we argue and show that the subtleties of ``even if...'' thinking are perhaps better captured in a casual setting.

% incorporate plausibility, robustness, and diversity into our semifactual framework.
% Whilst plausibility and diversity have been explored in semifactual explanation~\cite{kenny2021generating,artelt2022even}, robustness has not (indeed neither has causality).
% Moreover, successfully combining all three into one objective/algorithm is non-trivial, and yet to be demonstrated in the literature, yet we argue this is essential for maximizing positive outcomes in realistic settings.
% Next, we detail why these attributes are essential and how to instantiate them in semifactual explanation.

% We chose there three areas because they were the most popular at the time of writing (using Google Scholar keyword searches).
% Recall that there are other research areas in counterfactual recourse such as sequential decision making~\cite{naumann2021consequence,de2023synthesizing}, fairness~\cite{von2022fairness}, and privacy~\cite{pawelczyk2022privacy}, but we leave their full exploration within the framework of semifactual explanation for future work.

\paragraph{Plausible Gain}
We define plausibility here as explanations which are within distribution.
For example, an explanation saying a person could earn less and still have their loan accepted should change their ``debt-to-income ratio'' feature also, or it will lie outside the data manifold.
Prior work on semifactuals has only considered euclidean distance to training data as a heuristic for this~\cite{kenny2021generating}, in contrast we posit (similar to the counterfactual literature~\cite{karimi2019model}) that this is better approached with SCMs.
Hence, we define the plausibility for $\m{x}$ taking the action $a$ by
$P(\m{x}, a) = \mathrm{Pr}(a=do(\bm{\theta}) | \m{x})$
where $\m{x}$ is fixed for an individual and $\mathrm{Pr}(\cdot)$ is a density function.
In our non-causal tests, we use the $L_2$ norm to training data to approximate plausibility (i.e., being in distribution, similar to~\cite{laugel2019dangers,kenny2021generating,van2019interpretable}).
However, this issue of plausibility is naturally taken care of in our causal tests thanks to the SCM ensuring plausible feature mutations, so we don’t explicitly consider plausibility there going forward.

\paragraph{Robust Gain}
Continuing with the example of a person who has a loan accepted to buy a house, the semifactual should sometimes be robust to distribution shifts.
For example, if the person uses the semifactual explanation to triple their loan amount (recall~\cref{fig:title}), they will likely need upwards of six months to locate a new house during which the semifactual should hold if the person e.g. gets an additional credit card.
% Prior work in semifactuals has focused on robustness in learning procedures~\cite{lu2022rationale}, rather than individual explanation.
Hence, we define our semifactual robustness such that while taking action $a$, any close neighbor of the generated semifactual $\SF(\m{x}, a; \scm)$ can still receive a positive outcome.
% Next, we define an $\epsilon$-neighborhood of the individual $\m{x}$ which is denoted by $\nb(\m{x})$.
% We define that a datapoint $\bm{\theta}'$ can be reached by an individual $\m{x}$ with $\m{x} \hookrightarrow \bm{\theta}'$, if $\exists a \in \mathcal{A}(\m{x})$ such that $\bm{\theta}' = \SF(\m{x}, a; \scm)$.
% Then, $\m{x} \hookrightarrow \bm{\theta}'$ indicdates that there exists an \emph{actionable} feature change for $\m{x}$ to be transferred to $\bm{\theta}'$ through SCM $\scm$.
% \begin{definition}[$\epsilon$-Neighborhood]
The $\epsilon$-neighborhood of $\m{x}$ centering around an individual $\m{x}$ is
\begin{align}
\nb(\m{x}) = \{ \bm{\theta} = \SF(\m{x}, a; \scm): \forall a \in \mathcal{A}, \delta(\bm{\theta}, \m{x}) \le \epsilon \}
\end{align}
which covers all neighbors that can be reached from $\m{x}$ by taking an \emph{actionable} feature change $a$ through the SCM $\scm$.
% \end{definition}
% This indicates $\m{x}$ can transfer to $\bm{\theta}'$ via a certain action in the valid action set $\mathcal{A}^+(\m{x})$.
% This excludes non-actionable cases like \emph{``even if you were one year younger...''}, which is irrelevant as it is not reachable.
By definition, $\m{x}$ is also a neighbor of itself since $\m{x} \in \nb(\m{x}) $ holds given $\delta(\m{x}, \m{x})=0$.
Let us represent $\nb(\SF(\m{x}, a; \scm))$ by  $\nb_{s}(\m{x}, a)$ for simplicity.
Given the predictive model $h(\cdot)$ and an individual $\m{x}$, an action $a$ is robust for individual $\mathbf{x}$ if $h(\bm{\theta}) > \psi, \forall \bm{\theta} \in \nb_{s}(\m{x}, a)$, which is equivalent to
$\min_{\bm{\theta} \in \nb_{s}(\m{x}, a)} h(\bm{\theta}) - \psi > 0$.
For instance, $\psi=0.5$ works for a binary model case.
% As $\SF(\m{x}, a; \scm)$ is also an element in $\nb_{s}(\m{x}, a)$,~\cref{eq:robust} also ensures the semifactual itself will stay in the desired state.
% \end{definition}
We hence denote the term related to the robustness by
\begin{equation}
\label{eq:robust_term}
H(\m{x}, a) = \min_{\bm{\theta} \in \nb_{s}(\m{x}, a)} h(\bm{\theta}) - \psi \, ,
\end{equation}
which will be useful for constructing the final objective.

\paragraph{Diverse Gain}
It is generally preferred to offer a number of suggested actions $\{a_1, \dots, a_m\}$, rather than a single one~\cite{xin2022exploring}.
% This is in particular meaningful for semifactuals as these aim to offer users options for better managing users' budgets.
% In such a situation, we would like to increase the diversity among the suggestions.
% Additionally, the diversity among the suggestions can be equivalent to the diversity of the semifactuals generated by taking the suggested actions.
% Prior work in semifactuals has considered diversity~\cite{artelt2022even}, but did so by iterating each possible feature set (e.g., two actionable features would be 3 sets), which limits the final number of explanations possible.
Like prior work in counterfactuals, we define diversity as the average pair-wise distance among a set of entities~\cite{zhang2008avoiding,mothilal2020dice}.
We reuse the distance function $\delta$ and define the diversity objective within a set of SFs $\{\bm{\theta}_i\}_{i=1}^m \subseteq \mathcal{X}^m$ as
\begin{equation}
\label{eq:diversity}
R(\{\bm{\theta}_i\}_{i=1}^m) =
\begin{cases}
  \frac{2}{m(m-1)} {\sum_{i=1}^m \sum_{j > i}^m L_2 \circ \delta(\bm{\theta}_i, \bm{\theta}_j)} & m > 1 \\
  0 & m = 1
\end{cases}
\end{equation}
which represents a pairwise mean distance among the set of data points, based on the $L_2$ norm.
One may accommodate $m=1$ for the case when only a single semifactual is desired.

\subsubsection{Semifactual Objective}
The final objective may be constructed as follows.

\begin{definition}[Semifactual Objective]
\label{def:sf}
% Let $ \gamma$ be a positive regularizing parameter.
We consider a simple composition multiplication function for $f(\cdot)$.
Considering gain, plausibility, robustness, and diversity, the semifactual objective function is:
\begin{equation}
\label{eq:sf_obj}
\begin{aligned}
\max &\quad \frac 1 m \sum_{i=1}^m P(\m{x}, a_i)G(\m{x}, a_i) +  \gamma R\left( \{\SF(\m{x}, a_i; \scm)\}_{i=1}^m \right)  \\
\text{s.t.} &\quad \forall i=1,\dots, m: a_i \in \mathcal{A},  H(\m{x}, a_i) > {\psi} \, .
\end{aligned}
\end{equation}
In optimisation~\citep{boyd2004convex}, an adversarial interpretation from the perspective of a two-player zero sum game can further simplify~\cref{eq:sf_obj} to
\begin{equation}
\label{eq:final_obj}
\mathcal{J} \coloneqq \min_{\lambda_1, \dots, \lambda_m \ge 0}~ \max_{a_1, \dots, a_m \in \mathcal{A}}
\frac 1 m \sum_{i=1}^m P(\m{x}, a_i)G(\m{x}, a_i) + \lambda_i H(\m{x}, a_i) + \gamma R\left( \{\SF(\m{x}, a_i; \scm)\}_{i=1}^m \right) \, ,
\end{equation}
where $H(\m{x}, a)$ is the Lagrangian.
The primal player tries to maximise the plausibility-weighted gain and diversity, with regard to $a$, whilst the dual player tries to minimise regarding a set of $\lambda$.
\end{definition}
Since there are $m$ suggestions, the constraints for robustness will be $m$ times.
% Therefore, we have $|\{ \lambda_{p_1}, \lambda_{p_2}, \dots \} | = |\{ \lambda_{s_1}, \lambda_{s_2}, \dots \} | = m$, which respectively correspond to the constraints $\{H_p(\m{x}, a_i) \}_{i=1}^m$ and $\{H_a(\m{x}, a_i)\}_{i=1}^m$.
Observing the objective, the robustness is a hard constraint, whilst the diversity can be regarded as regularisation.
$P$ can be seen as a scaling factor for $G$ which helps to guarantee that high expected gain is only possible alongside high plausibility, simply adding them misses this special property.

\subsection{Properties of the Framework}
% \paragraph{Validation of the Relaxation.}
% First, we show that the probabilistic relaxation of the post-robustness is valid.
% \begin{proposition}
% \label{th:relax}
% Let us consider a special case $\tilde{\psi} = \psi$.
% $ \mathbb{E} [h(\bm{\theta}') | \m{x}, a] > \psi$ is the necessary but insufficient condition for $\min_{\bm{\theta}' \in \nb_{s}(\m{x}, a)}  h(\bm{\theta}') >  \psi$.
% \end{proposition}
% In other words, this indicates when $\tilde{\psi} = \psi$, the probabilistic robustness is always weaker than absolute robustness in the constraining power.
% See~\cref{sec:proof-relax} for the proof.

\paragraph{Effective Solution Space.}
We discuss the set of meaningful solutions here and the result validates the re-formulation [i.e., \cref{eq:final_obj}] of the semifactual framework [i.e., \cref{eq:sf_obj}].
First, we depict the lemma.
\begin{lemma}
\label{th:lower}
Assume that the limit of the gain function and diversity term are finite.
Also, assume that $\mathcal{A}^+ \coloneqq \{ a \in \mathcal{A}: G(\m{x}, a) \ge 0 \}$ is non-empty for an individual $\m{x}$.
The semifactual objective $\mathcal{J} \ge 0$ when $\forall i = 1, \dots, m, a_i \in \mathcal{A}^+: H(\m{x}, a_i) \ge 0$, otherwise $\mathcal{J}=-\infty$.
\end{lemma}
See~\cref{sec:proof-lower} for the proof.
We can summarise that the action set which is able to provide the positive payoff can be defined by the named effective solution space for $\m{x}$:
$\mathcal{A} = \{ a \in \mathcal{A}: H(\m{x}, a) \ge 0, G(\m{x}, a) > 0 \}$.
Hence, repeated suggestions will be produced when the number of actions in this space is smaller than the required $m$.
Otherwise, the solution will provide more versatile options.
There are no suggestions to achieve an effective semifactual(s) (i.e., with positive gain) if this solution space is empty.
However, similar situations exist for counterfactuals when they are also impossible to generate, assuming a similar set of actionable constraints are defined.

%% file: method.tex
% !TEX root = example_paper.tex

% Robust
% Diverse
% Semi
% Factual
% Recourse
% Explanation
% Contrastive

\section{Implementation Details}
We now introduce our methods to solve~\cref{eq:final_obj}, henceforth called Semifactual-recourse GENeration (S-GEN), for both causal and non-causal domains.
In the following paragraphs, we use $\hat{G}$ to denote an empirical approximation of $G$, and likewise for $P$, $H$, $R$, and $\mathcal{J}$.

\subsection{Causal Case}
Assuming the presence of a differentiable classifier $h(\cdot)$ and SCM $\scm$, (recall the latter guarantees plausibility), let $\Omega_i(\m{x}) = \{\mathrm{Pr}(\bm{\theta}): \bm{\theta} \in \mathcal{B}_s(\m{x}, a_i) \}$ be the probability distribution over the $\epsilon$-neighborhood of $\SF(\m{x}, a_i; \scm)$.
Also, let $\mathbf{B}_i$ represent a finite subset of $\mathcal{B}_s(\m{x}, a_i)$ sampled according to $\Omega_i(\m{x})$.
% Also, let $a \sim b$ represent that $a$ is drawn from distribution $b$.
% Like Dominguez et al.~\cite{dominguez2022adversarial}, we consider all categorical features to be real-valued.
Our objective is:
% \todo{Have to define what $\mathcal{L}$ is.}
\begin{align}
\label{eq:causalCase}
\max_{ a_1, \dots, a_m}
\min_{ \lambda_1, \dots, \lambda_m}
&\quad \frac 1 m \sum_{i=1}^m  - \lambda_i \mathcal{L} \left( h(\SF(\m{x}, a_i; \scm), h(\m{x}) \right) - \frac {1} {|\mathbf{B}_i|}\sum_{\bm{\theta}_i \in \mathbf{B}_i} \lambda_i \mathcal{L} \left( h(\bm{\theta}_i), h(\m{x}) \right) \nonumber \\
&\quad + \hat{P}(\m{x}, a_i) \hat{G}(\m{x}, a_i)   + \gamma \hat{R}( \{\SF(\m{x}, a_i; \scm)\}_{i=1}^m) \nonumber \\
\text{s.t.}
&\quad
\forall i = 1, \dots, m: a_i \in \mathcal{A}, \lambda_i > 0
\end{align}
where $\mathcal{L}$ is the binary cross entropy loss.
% \todo{At least one more sentence to describe what $\mathcal{L}$ does}
For robustness, we used Monte Carlo (MC) sampling with an epsilon $\epsilon$ robust hypersphere, and if either the instance or sampling return a \emph{negative outcome} with $h(\cdot)$, we use the prior optimisation step as the solution.
For diversity, $m$ is set to the number of actionable feature sets, and a solution is obtained for each.
% In the case $m$ is less than the number of actionable feature sets, the solution(s) best satisfying Eq.~\ref{eq:causalCase} are taken.
We utilise the causal recourse approach of Karimi et al.~\cite{karimi2020algorithmic} for solving the maximin.
The actionable bounds are clipped each iteration, and $\lambda$ is iteratively decreased to put more emphasis on gain over time (see Algorithm~~\ref{alg:sg}).

\subsection{Non-Causal Case}
\label{sec:nonCausalImplementation}
For the non-causal case, we use a genetic algorithm~\citep{virgolin2022robustness,whitley1994genetic} which only assumes a binary predictive model $h(\cdot)$.
This approach follows the standard design for genetic algorithms, with some minor alterations specifically for semifactual generation, see~\Cref{app:algorithm} for the pseudocode.
Next we present the fitness function which optimises our objective.

\subsubsection{Fitness Function}
\label{sec:fitfunction}
For gain, the average distance between an individual $\m{x}$ and each semifactual $\SF(\m{x}, a)$ is measured as $ \hat{G}(\m{x}, a) = \|\SF(\m{x}, a) - \m{x}\|_2$.
% We attempt to maximize this distance to ensure people can gain the most from their recourse option(s).
% \paragraph{Robustness.} There are two parts to the constraints here, satisfying the $\epsilon$-neighborhood property that $\forall \m{x}' \in \nb_{s}(\m{x}, a): h(\m{x}') = h(\m{x}) = 1$, and ensuring the actual semifactual is classified correctly as $h(\m{x})$ also.
For robustness, we relax it to two constraints: $H_p$ is the probabilistic robustness for the neighbor points where the generated semifactuals for a query are randomly perturbed using MC simulation to make sure the surrounding neighborhood is robust, and $H_a$ the absolute robustness for the individual $\m{x}$ (more detail in~\cref{sec:prob_relaxation}).
% Let $y \sim z$ denote $y$ is distributed by $z$.
For the first constraint, a score of $\hat{H}_{p}(\m{x}, a) = \frac{1}{n}\sum_{i=1}^{n}\mathbbm{1}\{ h(\m{x}) = h(\bm{\theta}_i) \}$ where $\bm{\theta} \sim \mathrm{Pr}(\nb_s(\m{x}, a))$,  is returned.
For the second constraint, a score of $\hat{H}_{a}(\m{x}, a) = \mathbbm{1}\{ h(\m{x}) = h \circ \SF(\m{x}, a)\}$ is returned.
Hence, the solution is rewarded for (1) the neighborhood samples, and (2) the semifactuals themselves being classified as $h(\m{x})$.
% One feature is perturbed at a time, either by adjusting a categorical variable to be another actionable value, or perturbing a continuous feature by sampling from a standard normal distribution using variance $\sigma^2$, where we set $\sigma^2$ to 5\% of the actionable feature range.
% This is repeated for many iterations on each semifactual $\SF(\m{x}, a)$ to generate a set of $n$ perturbations $\{\m{x}_i'\}_{i=1}^{n}$, and an objective $\hat{H}_{p}(\m{x}, a) = \frac{1}{n}\sum_{i=1}^{n}\mathbbm{1}\{ h(\m{x}) = h(\m{x}'_i) \}$ is returned for each semifactual, where $\mathbbm{1}(\cdot)$ is the indicator function.
% For the second constraint, the objective is calculated as $\hat{H}_{a}(\m{x}, a) = \mathbbm{1}\{ h(\m{x}) = h(\SF(\m{x}, a))\}$ given a specific action $a$, thus the solution is rewarded by all aspiring semifactuals being on the correct side of the decision boundary.
For plausibility, we take from prior work and directly use the training data~\cite{kenny2021generating}.
Specifically, considering the training data set $\mathcal{D}$, we define the notion of plausibility using the distance of each semifactual generated to the nearest training data point.
As the term must be maximised, we use a function which is monotonically decreasing with respect to the distance with $P(\m{x}, a) \approx \hat{P}_{\mathcal{D}}(\m{x}, a) = \exp\{ 1 / ( \min_{\bm{\theta} \in \mathcal{D}} \| \SF(\m{x}, a) - \bm{\theta} \|_2^2 + \gamma_p ) \}$
where  $\gamma_p$ is to account for when a perfect match to the semifactual exists in the training data (thus the division is undefined), and $\hat{P}_{\mathcal{D}}(\m{x}, a)$ is an empirical approximation (based on $\mathcal{D}$) of the plausibility.
% Hence, if the semifactual is close to a datapoint $\m{d}$, the output will be large, and likewise, if it is far the output will be small.
Lastly, for diversity~\cite{zhang2008avoiding,mothilal2020dice}, we take the mean distance between all $m$ generated semifactuals with $\hat{R}( \{\SF(\m{x}, a_i)\}_{i=1}^m)$ which precisely follows~\cref{eq:diversity}.
This objective collapses to 0 when $m=1$.
% Notably, this term allows the semifactuals to be as diverse as possible, but the gain objective encourages them to also be far from the query $x$.

Certain objectives need to be weighted individually based on the problem.
For example, explanations which can be acted upon immediately perhaps don't need robustness.
% Notably, the $\lambda_{p_i}$ and $\lambda_{s_i}$ multipliers in~\cref{eq:final_obj} are relaxed to for one hyperparameter each, and are used alongside $ \gamma$ to balance the objectives.
Notably, the multiplier $\lambda$ in~\cref{eq:final_obj} is split to $\lambda_p$ and $\lambda_a$, for $H_p$ and $H_a$ respectively.
In this work, we treat them as hyperparameters.
Also, they are used alongside $ \gamma$ to balance the objectives.
Since $\lambda$ and $ \gamma$ are selected as hyperparameters, they are removed under the min operator.
Finally, the objective (fitness) function is defined as:
\begin{align*}
% \label{eq:final_obj2}
\max_{ a_1, \dots, a_m \in \mathcal{A}^+}
\frac 1 m \sum_{i=1}^m \hat{P}_{\mathcal{D}}(\m{x}, a_i) \hat{G}(\m{x}, a_i) + \lambda_{p} \hat{H}_{p}(\m{x}, a_i)
+ \lambda_{s} \hat{H}_{a}(\m{x}, a_i) + \gamma \hat{R}( \{\SF(\m{x}, a_i; \scm)\}_{i=1}^m) \,.
\end{align*}
%
% Note we write $\hat{\mathcal{J}}$ instead of $\mathcal{J}$ in the implementation to indicate that it is an empirical approximation of the objective.
We selected the hyperparameters via a grid search, see~\cref{sec:hyperApp}.
Crucially, we also weight the fitness function output by $\hat{H}_{p}(\m{x}, a)$ to encourage solutions with more semifactuals (see Algorithm~~\ref{alg:noncausal}).

%% file: appendix.tex
\newpage

\section{Property Analysis}
\label{appendix:proofs}
% \subsection{Proof of~\cref{th:relax}}
% \label{sec:proof-relax}
% \begin{proof}
% Let $p(\bm{\theta})$ denote the density function of $\bm{\theta} \sim P_{\nb_\SF(\m{x}, a)}$.
% With regard to the necessity, we write
% \begin{align*}
% \mathbb{E} [h(\bm{\theta}) | \m{x}, a]
% &= \int p(\bm{\theta}) h(\bm{\theta}) d \bm{\theta} \\
% &\ge \int p(\bm{\theta}) \min_{\bm{\theta} \in \nb_{s}(\m{x}, a)} h(\bm{\theta}) d \bm{\theta} \\
% &= \min_{\bm{\theta} \in \nb_{s}(\m{x}, a)}  h(\bm{\theta}) \,.
% \end{align*}
% Therefore,
% \begin{align*}
% \min_{\bm{\theta} \in \nb_{s}(\m{x}, a)}  h(\bm{\theta}) > \rho
% \implies \underset{\bm{\theta} \in \nb_{s}(\m{x}, a)}{\mathbb{E}}  h(\bm{\theta}) >  \rho \,.
% \end{align*}
% On the other hand, we assume the sufficiency holds, i.e.,
% \begin{align*}
% \underset{\bm{\theta} \in \nb_{s}(\m{x}, a)}{\mathbb{E}}  h(\bm{\theta}) > \rho
% \implies \min_{\bm{\theta} \in \nb_{s}(\m{x}, a)}  h(\bm{\theta}) >  \rho \,.
% \end{align*}
% Take a special case of that only two $\bm{\theta}$ are available and they have the even probabilities.
% The two $h(\bm{\theta})$ can be $ \rho + c$ and $ \rho - c$ for a suitable positive constant $c$.
% Such a $c$ does not exist if and only if $h(\cdot)$ has one possible output which is definitely not considered in our task.
% Hence, the minimal $h(\bm{\theta})$, which is $\rho - c$, is smaller than $\rho$, which contradicts the assumption.
% The proof finishes here.
% \end{proof}

\subsection{Proof of~\cref{th:lower}}
\label{sec:proof-lower}
\begin{proof}
We can rewrite $\mathcal{J}$ to
\begin{align}
\label{eq:final_obj2}
\mathcal{J} &= \max_{a_1, \dots, a_m \in \mathcal{A}} \left \{ \frac 1 m \sum_{i=1}^m  P(\m{x}, a_i) G(\m{x}, a_i) + \frac 1 m \sum_{i=1}^m \min_{\lambda_{i} \ge 0} \lambda_i H(\m{x}, a_i) + \gamma R\left( \{\SF(\m{x}, a_i; \scm)\}_{i=1}^m \right) \right \} \, .
\end{align}
We derive the fact that, for any $i$,
\begin{align*}
\min_{\lambda_i \ge 0} \lambda_i R(\m{x}, a_i)
&=
\begin{cases}
-\infty & H(\m{x}, a_i) < 0 \\
0 & \mbox{otherwise}
\end{cases}
\end{align*}
where $-\infty$ comes from setting $\lambda_{i} = \infty$ and $0$ is obtained by setting $\lambda_{p_i}=0$.
By the linearity of summation, we can further derive
\begin{align*}
\frac 1 m \sum_{i=1}^m \min_{\lambda_i \ge 0} \lambda_i R(\m{x}, a_i)
&=
\begin{cases}
-\infty & \exists i, H(\m{x}, a_i) < 0 \\
0 & \mbox{otherwise}
\end{cases} \, .
\end{align*}
% We then achieve, by rewriting the results in an equivalent representation,
% \begin{align*}
% \frac 1 m \sum_{i=1}^m \min_{ \lambda_{s_i} \ge 0} \lambda_{s_i} \mathcal{J}_{r_s}(\m{x}, a_i)
% &=
% \begin{cases}
% 0 & \forall i, \mathcal{J}(\m{x}, a_i) \ge 0 \\
% -\infty & \mbox{otherwise}
% \end{cases} \, .
% \end{align*}
That is, if any constraint for the robustness is unsatisfied, the dual player will minimise the objective towards $-\infty$; however, the primal player cannot optimise towards $\infty$ given that the limit of the gain function and the diversity are finite.
In other words, if the constraints are satisfied, the primal player can freely optimise the objective.
Once $H(\m{x}, a_i) \ge 0, \forall a_i$ are satisfied, the objective becomes
\begin{equation}
\begin{aligned}
\label{eq:final_obj2}
\tilde{\mathcal{J}}
&\coloneqq \max_{a_1, \dots, a_m \in \mathcal{A}^+} \frac 1 m \sum_{i=1}^m P(\m{x}, a_i) G(\m{x}, a_i) + \gamma R\left( \{\SF(\m{x}, a_i; \scm)\}_{i=1}^m \right) \\
&\ge \min_{a_1, \dots, a_m \in \mathcal{A}^+} \frac 1 m \sum_{i=1}^m P(\m{x}, a_i) G(\m{x}, a_i) + \gamma R\left( \{\SF(\m{x}, a_i; \scm)\}_{i=1}^m \right) \\
&> 0
\end{aligned}
\end{equation}
as $P(\m{x}, a_i) > 0$ and $G(\m{x}, a_i) \ge 0$ for any $a_i \in \mathcal{A}^+$; also, $R \ge 0$ holds.
We conclude the proof here.
\end{proof}

\subsection{A Probabilistic Relaxation of Robustness}
\label{sec:prob_relaxation}
Absolute robustness is difficult to guarantee, and common practice is to relax this via a probabilistic approach~\citep{karimi2020algorithmic}.
% The interpretation is that one can only guarantee that ones loan will be accepted with good chances when the situation sightly changes, as uncertainty lives in the realworld.

Assume there is a distribution over the sample space $\nb_{s}(\m{x}, a)$ denoted by $\mathrm{Pr}(\nb_s(\m{x}, a))$.
We write $\bm{\theta} \sim \mathrm{Pr}(\nb_s(\m{x}, a))$ to indicate that $\bm{\theta}$ is sampled from the set $\nb_{s}(\m{x}, a)$ under the density $\mathrm{Pr}(\cdot)$.
Let $\mathbb{E} [h(\bm{\theta}) | \m{x}, a]$ denote the expectation of $\bm{\theta}$ in this configuration.
Hence, we modify~\cref{eq:robust_term} to
\begin{align}
\label{eq:exp-robust}
\mathbb{E} [h(\bm{\theta}) | \m{x}, a] &> \tilde{\psi},
\end{align}
where $\tilde{\psi}$ is a function that adjusts the base score threshold $\psi$.
It is crucial to have this threshold function in order to consider the variance of scores in the neighbor set.
Particularly, we would like most neighbors to remain in a similarly ``good" state, with low variance between them.

Moreover, we explicitly impose $ h(\SF(\m{x}, a; \scm)) - \psi > 0 $.
It places a hard constraint to avoid the case in which the neighbors of the semifactual are robust, but the ``semifactual" itself has crossed the decision boundary to become a counterfactual.
Whilst somewhat unlikely, this situation is theoretically possible, and requires consideration.
In this case, $H$ is re-written as $H_{p}$, which represents a combination of (i) the probabilistic robustness, and (ii) the absolute robustness for the semifactual $H_a$ such that:
\begin{equation}
H_{p}(\m{x}, a) = \mathbb{E} [h(\bm{\theta}) | \m{x}, a] - \tilde{\psi} \qquad
H_a(\m{x}, a) = h(\SF(\m{x}, a; \scm)) -  \psi \,.
\end{equation}
% Combining the robustness,~\cref{eq:sf-obj3} is extended to
% \begin{align}
% \label{eq:sf-obj3-plus}
% \min_{\lambda_{p}, \lambda_{s} \ge 0} \max_{a \in \mathcal{A}^+} G(\m{x}, a) + \lambda_p H_{p}(\m{x}, a) + \lambda_{s} H_a(\m{x}, a) \, .
% \end{align}
In practice, $H_{p}$ is still non-trivial to solve.
Monte Carlo (MC) sampling is a common strategy to apply here such that, by sampling a fixed sized batch $\m{B} = \{\bm{\theta}: \bm{\theta} \sim \mathrm{Pr}(\nb_s(\m{x}, a)) \}$,
\begin{equation}
\textstyle H_{p}(\m{x}, a) = \mathbb{E} [h(\bm{\theta}) | \m{x}, a] - {\tilde{\psi}}
\approx (1 / |\m{B}|) \sum_{\bm{\theta} \in \m{B}}  {h(\bm{\theta})} - {\tilde{\psi}} \,.
\end{equation}
This implies that we substitute an unbiased estimator for the population mean.

\section{Actionability Constraints}
\label{sec:actConsts}

\subsection{Non-Causal}
Here we define the actionability constraints used in the various domains.
It may be assumed that the direction features are allowed to change corresponds with \textit{positive gain}.
% \textit{Recall that what is actionable in a counterfactual context, is not necessarily actionable in a semifactual one.
% }
% In general, it is intuitive to feel that far more features are actionable from a semifactual perspective.
% For example, it is reasonable that someone may want to take a demotion to a less well paid job, to have more time with their family etc.
% It is not however likely actionable that someone could choose to be promoted at will.
% We ordered categorical features in a sensible fashion which represented such reasoning, and when we say a categorical feature could decrease/increase, we are referring to this pre-defined order.
% If you are interested in the exact ordering, please refer to our code which contains all the lists, but here we summarize.
% In reality however, a user must specify their exact actionability constraints, what we have specified here is designed to be representative what is possible for the ``average" individual.
We use various sized ``action sets" to fully test all algorithms in various setups.
The German Credit data used 15 actionable features to be closely in line with Mothilal et al.~\cite{mothilal2020dice} whom allowed all features to be mutable.
However, we also used 7 on Lending Club, and 4 on Adult Census/Breast Cancer to test the algorithms in situations with smaller action spaces also for completeness.

We ordered categorical features in a sensible fashion to ``direct" semifactual ``even if" thinking, and when we say a categorical feature could decrease/increase, we are referring to this pre-defined order.
If you are interested in the exact ordering, please refer to our code which contains all the lists, but here we summarise.
In reality however, a user must specify their exact actionability constraints, what we have specified here is designed to be representative what is possible for the ``average" individual.

\subsubsection{German Credit Dataset}
The continuous features used were `duration', `amount', `age', the categorical ones were `status', 'credit\_history', `purpose', `savings',
       `employment\_duration', `installment\_rate', `personal\_status\_sex',
       `other\_debtors', 'present\_residence', `property',
       `other\_installment\_plans', `housing', `number\_credits', `job',
       `people\_liable', `telephone', `foreign\_worker'.
As actionable features for semifactual recourse, we considered the following:
\begin{itemize}
  \item \textit{duration}: We allowed people to increase the duration of their loan.
  \item \textit{amount}: We allowed people to increase the amount of their loan.
  \item \textit{status}: We allowed people to move towards having lower status.
  \item \textit{credit\_history}: We allowed people to move towards e.g. having a late payment if their credit history was otherwise good.
  \item \textit{savings}: This feature was allowed to decrease.
  \item \textit{employment\_duration}: This feature was allowed to decrease in case people wanted to e.g. start a new job.
  \item \textit{installment\_rate}: This feature was allowed to move towards lower payments.
  \item \textit{other\_debtors}: this feature was allowed to add another co-applicant.
  \item \textit{present\_residence}: This feature was allowed to move towards e.g. renting in case the user desired to do so whilst searching for a new house with their loan.
  \item \textit{property}: this feature was allowed to move towards having no property in case the user desired to sell their house/car etc to help pay for e.g. a downpayment.
  \item \textit{other\_installment\_plans}: This feature was allowed to add other installment plans.
  \item \textit{housing}: this feature was allowed to move towards renting away from e.g. owning.
  \item \textit{number\_credits}: This feature was allowed to increase if the user desired to acquire more credit cards.
  \item \textit{job}: this feature was allowed to decrease in case the individual desired to get a different, less demanding job within their institution, or indeed quite their job to e.g. start a business.
  \item \textit{people\_liable}: This feature was allowed to move towards more people being liable.
\end{itemize}

\subsubsection{Lending Club}
The continuous features used were `loan\_amnt', `pub\_rec\_bankruptcies', `annual\_inc', `dti', the categorical ones were `emp\_length', `term', `grade', `home\_ownership', `purpose'.
As actionable features for semifactual recourse, we considered the following:
\begin{itemize}
  \item \textit{home\_ownership}: This feature was allowed to decrease towards e.g. renting.
  \item \textit{annual\_inc}: this feature was allowed to decrease if the person desired to e.g. work less hours.
  \item \textit{emp\_length}: This feature was allowed to decrease in case the individual desired to change careers.
  \item \textit{dti}: dept to income ratio, this feature was allowed to increase.
  \item \textit{pub\_rec\_bankruptcies}: This feature was allowed to increase in case the user decided they wanted to declare bankruptcy to e.g. try and keep some assets.
  \item \textit{loan\_amnt}: this feature was allowed to increase.
  \item \textit{term}: This feature was allowed to decrease.

\end{itemize}

% \subsubsection{Adult Census}
% The continuous features used were `age', `capital-gain', `capital-loss', `hours-per-week', the categorical ones were `education', `educational-num', `gender', `marital-status',
%        `native-country', `occupation', `race', `relationship', `workclass'.
% As actionable features for semifactual recourse, we considered the following:
% \begin{itemize}
%   \item \textit{occupation}: This feature was allowed to move towards what are generally considered lower skilled work the user may move into if they desire.
%   \item \textit{capital-gain}: this feature was allowed to decrease.
%   \item \textit{capital-loss}: This feature was allowed to increase.
%   \item \textit{hours-per-week}: this feature was allowed to decrease in case they want to e.g. spend more time with their family.
% \end{itemize}

\subsubsection{Breast Cancer}
The continuous features used were none, the categorical ones were `agegrp', `density', `race', `Hispanic', `bmi', `agefirst', `nrelbc',
       `brstproc', `lastmamm', `surgmeno', `hrt'.
As actionable features for semifactual recourse, we considered the following:
\begin{itemize}
  \item \textit{bmi}: This feature was allowed to move towards less healthy BMI levels in case the patient e.g. has hypothyroidism.
  \item \textit{brstproc}: this feature was allowed to move towards having had a previous breast proceedure in case the patient would like to do so or was advised.
  \item \textit{hrt}: This feature was allowed to move towards starting HRT, in case a person may wish to alleviate synthoms of the menopause.
  \item \textit{agegrp}: this feature was allowed to get older in case the individual would like to take no action confident that it would not lead to cancer in the next few years/decades.
\end{itemize}

% As this is a medical domain were readers are less likely to be experts these choices likely require more explanation.
% The motivation for allowing age to mutate is so that people can be told e.g. \textit{``Even if you don't take any action, you will not get cancer in the next 5-10 years". }
% This way, people are essentially given the information that even if they don't make drastic changes to their lifestyle, they may be confident such action will still lead to good health.
% We allowed people to be told it was ok to start hormone replacement therapy also, so people may be told \textit{``Even if you wish to start HRT, it will not cause cancer"}, which is important as such therapy increases the risk of breast cancer, but is useful to decrease synthoms of the menopause, which some people may like to do in the knowledge they will be safe.
% Breast procedure (brstproc) was allowed to change to \textit{``having had a previous breast procedure"}, as such procedures are associated with increased cancer risk and individuals may wish to have one in confidence.
% BMI outside the normal range is associated with an increased breast cancer risk, so we allowed this to change to more ``at risk" categories in case people wish to exercise their personal freedom to do so in the confidence it will not drastically increase their cancer risk.

\subsection{Causal}
\label{sec:actConstsCausal}
In the causal setting, we allowed a user's age to increase a maximum of 5 years to mimic the motivating examples in the paper about a user having a bank loan accepted.
In such a situation, the user may want to e.g. work less hours over the next 5 years whilst they repay the loan, and still have it accepted.

Next, we detail the direction features are allowed to change, and what direction corresponds to \textit{positive gain}.

\subsubsection{Adult Income Census}
We use the features “sex”, “age”, “native-country”, “marital-status”, “education-num”, “hours-per-week”, which are the variables in the causal graph of Nabi \& Shpitser~\cite{nabi2018fair}.
We consider “age” and ``hours-per-week'' as actionable.
We allow ``age'' to increase a maximum of five years, and ``hours-per-week'' to decrease.

For positive gain, we considered: Age, marital status, and eduation-num \textit{increasing} corresponding to positive gain, and hours-per-week \textit{decreasing} corresponding to positive gain.
A persons sex was seen as neutral gain.

\subsubsection{COMPAS}
We use the features “age”, “race”, “sex” and “priors count”, which are the variables in the causal graph of Nabi \& Shpitser~\cite{nabi2018fair}. We consider ``age'' and “priors count” as actionable. As actionability constraints, we assume that both features are non-negative and can only be increase.
Age specifically is only allowed to increase by 5 years for each individual.

For positive gain, we considered: Age and priors count increasing corresponding to positive gain.
A persons sex and race was seen as neutral gain.

\section{Hyperparameter Choices}
\label{sec:hyperApp}
In this section, we discuss the hyperparameter specifications for the causal and non-causal cases respectively.

\subsection{Non-Causal}
Here we note the values for the hyperparameters used in our demonstrations.
All were obtained though pilot grid-searches across each dataset.
The hyperparameter choices are summarised in~\cref{tb:hyper}
\begin{table}[h]
\centering
\caption{Hyperparameter Specifications}
\label{tb:hyper}
\begin{tabular}{c c c c c }
\toprule
Data & $\lambda_p$ & $\lambda_{s}$ & $\gamma_d$ &  $\gamma_p$ \\
\midrule
German credit & 30 & 10 & 1 & $1e^{-1}$ \\
Lending Club  & 30 & 10 & 1 & $1e^{-1}$ \\
% Adult Census  & 50 & 10 & 1 & $1e^{-1}$ \\
Breast Cancer & 10 & 10 & 10 & $1e^{-1}$ \\
\bottomrule
\end{tabular}
\end{table}

% In the German credit dataset we used $\lambda_{p}=30$  $\lambda_{s}=10$  $\gamma_d=1$

% In the Lending Club dataset we used $\lambda_{p}=30$  $\lambda_{s}=10$  $\gamma_d=1$

% In the Adult Census dataset we used $\lambda_{p}=50$  $\lambda_{s}=10$  $\gamma_d=1$

% In the Breast Cancer dataset we used $\lambda_{p}=10$  $\lambda_{s}=10$  $\gamma_d=10$

For S-GEN itself, we used the same hyperparameters everywhere outside of the above table.
The number of generations spent searching for a solution was 20.
The population size was fixed at $\{12,24,48,72,96,120\}$, for diversity sizes of $\{1,2,4,6,8,10\}$, respectively.
The mutation rate was 0.05.
The number of ``elite" solutions passed on for each generation was 4.
The probability of a crossover happening was 0.5.
The number of Monte Carlo trials for each instance was 100.
The continuous features were perturbed (in mutation or population initialization) by the output from sampling a standard normal distribution with standard deviation equal to the max actionable feature value, minus the min actionable feature value, multiplied by 0.05.

\subsection{Causal}
In our causal tests we chose $\lambda$ as 1.0, and this was gradually decreased by a momentum of $\eta$=0.9 each iteration to put more emphases on the maximization of gain.

\section{Algorithm Pseudocode}
\label{app:algorithm}

% This section details our genetic algorithm for additional clarity and convenience.

\begin{algorithm}[h]
\begin{algorithmic}[1]
\REQUIRE $\m{x}$ the user feature
\REQUIRE $h(\cdot)$ the predictive model
% \REQUIRE $S_{\scm}(\cdot, \cdot)$ the semifactual function
\REQUIRE $m$ the expected number of suggestions
\REQUIRE $n$ the number of candidates, $n > m$
\ENSURE $\m{R}_{SF}$ the set of semifactual(s)
\STATE Sample $n$ candidates $\m{D} \gets \{\bm{\theta}_i \sim \mathcal{X}\}_{i=1}^{n}$
\WHILE{the stopping criterion is not satisfied}
\STATE Obtain the fitness scores $\m{f}$ with respect to $\m{D}$
\STATE Save the fittest $\bm{\theta}^* \in \m{D}$ according to $\m{f}$
\STATE Let $\m{D}$ evolve by \emph{natural selection} according to $\m{f}$, \emph{crossover}, \emph{mutation}, and \emph{elitism} with $\m{x}^*$
\ENDWHILE
\STATE Collect the best $m$ unique candidates from $\{\bm{\theta} \in \m{D}: h(\bm{\theta}) = h(\m{x}) = 1 \}$ to $\m{R}_{SF}$, according to the corresponding fitness scores in $\m{f}$
% \STATE Filter $\m{R}_{SF} \gets \{\bm{\theta} \in \m{R}_{SF}: h(\bm{\theta}) = h(\m{x}) = 1 \}$.
\IF{$|\m{R}_{SF}| < m$}
\STATE Complement $\m{R}_{SF}$ to $m$ elements with $\bm{\theta}$ randomly drawn from $\m{R}_{SF}$
\ENDIF
\end{algorithmic}
\caption{S-GEN: Genetic Algorithm to Generate semifactual Recourse with Robustness and Diversity in a Non-Causal Model Agnostic Setting}
\label{alg:noncausal}
\end{algorithm}

\begin{algorithm}[t]
\begin{algorithmic}[1]
\REQUIRE $\m{x}$ the user feature vector
\REQUIRE $h(\cdot)$ the predictive model
\REQUIRE $\scm$ the \textbf{differentiable} SCM
\REQUIRE $\epsilon$ the epsilon robustness
\REQUIRE $\eta$ the momentum parameter
\REQUIRE $\tau$ the learning rate
\REQUIRE $\mathrm{Proj}(\cdot)$ a projection function that ensures the action is actionable
\ENSURE $\m{R}_{SF}$ the set of semifactual(s)
\STATE $\m{R}_{SF} \gets \emptyset$
\STATE $i \gets 0$
% \STATE Initialize $a_1, \dots, a_m$
\FOR{$a \in \mathcal{A}$}
\STATE Move to next loop if the SF generated with the initial $a$ does not satisfy the constraints.
% \STATE \COMMENT{Check if the initial action $a$ itself is a valid semif-factual}
% \STATE Sample a batch of neighbors from $\mathbb{B}_s(\m{x}, a)$, denoted by $\mathbf{B}$
% \IF{$h(\SF(\m{x}, a; \scm)) = 0$ or $h(\bm{\theta}) = 0, \exists \bm{\theta} \in \mathbf{B}$}
% \STATE \textbf{break}
% \ENDIF
% \STATE $a_i \gets 0$
% \STATE $a_i' \gets 0$
% \STATE $\delta_i \gets 0$
\STATE $a_i \gets a$
\WHILE{not converged}
\STATE Sample a batch of neighbors from $\mathbb{B}_s(\m{x}, a_i)$, denoted by $\mathbf{B}_i$
\IF{$h(\SF(\m{x}, a_i; \scm)) = 0$ or $h(\bm{\theta}) = 0, \exists \bm{\theta} \in \mathbf{B}_i$}
\STATE \textbf{break}
\ENDIF
% \STATE $a_i \gets a_i'$
\STATE $\mathcal{J}_i \gets - \lambda_i \mathcal{L} \left(h(\SF(\m{x}, a_i; \scm)), h(\m{x}) \right) - \sum_{\bm{\theta}_i \in \mathcal{B}_i} \frac{\lambda_i}{|\mathbf{B}_i|}  \mathcal{L} \left(  h(\bm{\theta}_i), h(\m{x}) \right) + \hat{P}(\m{x}, a_i)\hat{G}(\m{x}, a_i) $
\STATE $a_i \gets \mathrm{Proj} \left( a_i + \tau \nabla_{a_i} \mathcal{J}_i \right)$
\STATE $\lambda_i \gets \eta \lambda_i$
\ENDWHILE
\STATE $\m{R}_{SF} \gets \m{R}_{SF} \cup \{ \SF(\m{x}, a_i; \scm) \}$
\STATE $i \gets i + 1$
\IF{$i \ge m$}
\STATE \textbf{break}
\ENDIF
\ENDFOR
\end{algorithmic}
\caption{S-GEN: Algorithm to Generate Robust \& Diverse Causal semifactual Explanations for Differentiable Classifiers}
\label{alg:sg}
\end{algorithm}

\section{Code}
\label{sec:code}
For our full code used please see:

\url{https://github.com/EoinKenny/Semifactual_Recourse_Generation}

The ability to reproduce the results is given.

\newpage

\section{Individual Dataset Results}
\label{sec:full_figs}
The results are presented in~\Cref{fig:diverse}.

\begin{figure}[!ht]
  \centering
  \includegraphics[width=0.95\textwidth]{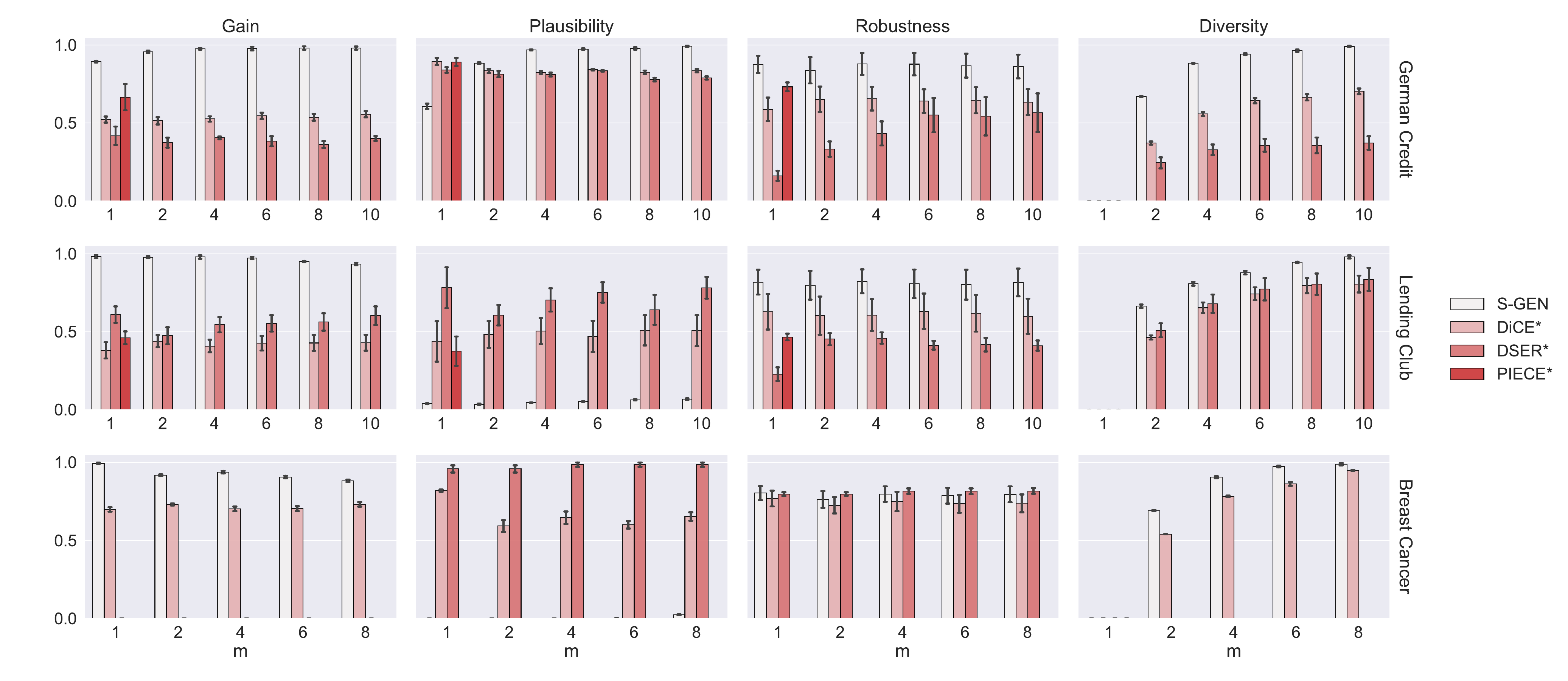}
  \caption{
  Results:
  The ability of S-GEN to create semifactuals is compared to DiCE* and PIECE*.
  Overall, S-GEN does the best, achieving significantly better results to both baselines on 11/16 tests.
  Moreover, S-GEN was only significantly worse than either baseline on a single test (i.e., plausibility on German Credit), with the remaining four tests being competitive between methods.
  % The most notable discrepancy to this occurs on the plausibility metric for the German Credit dataset, likely because this dataset was the smallest (1000 instances), and thus it is simply harder to find nearest neighbor matches for this metric, although S-GEN still performed best here at $m=1$.
  Standard error bars are shown.
  }
  \label{fig:diverse}
\end{figure}

% \label{sec:appQualAnal}

% \begin{table}[!h]

% \caption{The average results for all metrics across all datasets. Notably, S-GEN performs best in every metric in both single and diverse semifactual generation.}
% \centering
% \begin{tabular}{@{}llrrl@{}}
% \toprule
% Method & \multicolumn{1}{c}{Gain} & \multicolumn{1}{c}{Plausibility} & \multicolumn{1}{c}{Robustness} & Diversity \\ \midrule
%        & \multicolumn{4}{c}{Single semifactual Generation}                                                       \\ \midrule
% S-GEN   & \textbf{2.33}                     & \textbf{0.71}                             & \textbf{0.98}                           & N/A       \\
% DiCE*  & 1.38                     & 1.10                             & 0.84                           & N/A       \\
% PIECE* & 1.71                     & 1.14                             & 0.87                           & N/A       \\ \midrule
%        & \multicolumn{4}{c}{Diverse semifactual Generation}                                                      \\ \midrule
% SGEN   & \textbf{2.32}                     & \textbf{0.93}                             & \textbf{0.97}                           & \textbf{2.16}      \\
% DiCE*  & 1.47                     & 1.15                             & 0.88                           & 1.55      \\ \bottomrule
% \end{tabular}
% \label{table:avgAppendix}
% \end{table}

\section{Baselines}
\label{sec:baselinesAppendix}
\subsection{Non-Casual}

\paragraph{DiCE}
Our modification to DiCE, starts by generating a counterfactual(s) for a query.
Next, we use the algorithm again, but on the generated counterfactuals(s), to make them generate a second counterfactual, which goes back over the decision boundary.
In effect, this generates a semifactual(s) for a query.

\paragraph{PIECE}
Second, we use the PIECE framework by Kenny and Keane~\cite{kenny2021generating}, but apply it to tabular data.
Following the authors, we divide the training data into two sets, the first corresponding to those predicted as the original class $c$, and the second to those predicted as the counterfactual class $c'$, these are again split into the respective features.
Hence, if there are 2 classes, with 4 features, there are $2 \times 4=8$ sets of data.
These sets were then modeled using the best fit found for a Beta distribution on continuous features, and a simple Categorical distribution for categorical features.
To generate a semifactual predicted as $c$, we take the probability of each feature value in the query using the models of the counterfactual class $c'$, and modify each to be its expected statistical value in $c'$ one-by-one (from the lowest probability to the highest), until the next would take it over the decision boundary.
In the case of continuous features, as done by Kenny and Keane~\cite{kenny2021generating}, we take the probability as being the minimum of the two integrals either side of the feature value in the distribution.
In the case the expected feature values lie outside the actionability range, we clip them to the closest value allowed.

% Notably, DiCE* is not guaranteed to generate semifactuals under this setup, but S-GEN and PIECE* are, provided their action space is not empty (see~\cref{th:lower}).

\paragraph{DSER}
For Diverse Explanation of Reject~\cite{artelt2022even} (DSER) we had to modify the the technique in two main ways.
Most notably, the techniques doesn't deal with categorical features, so to overcome this, we optimised treating all one hot encoded features as real-valued, and then projected each categorical feature onto its nearest value.
Next, the method addresses diversity by iterating all different sets of possible features, in our domains this is computationally intractable.
Hence, we optimise one semifactual at a time, each time pushing each solution as far as possible from those already found.

\subsection{Causal}
\paragraph{Karimi et al.~(2021)}
The method by Karimi et al.~\cite{karimi2021algorithmic} is a recourse method designed to minimise cost whilst traversing the decision boundary.
To modify the technique, we simply stop the optimization when the next step would take it over the decision boundary.

\paragraph{Dominguez et al.~(2022)}
The method by Dominguez et al.~\cite{dominguez2022adversarial} is identical to Karimi et al.~\cite{karimi2021algorithmic}, but they add in a robustness component.
Namely, they take an individual $x$, and solve an inner loss which means that an individual of distance $\epsilon=0.1$ (in our tests) close to $x$, with the same recourse given, will also achieve recourse.
We simply keep the same optimization process, but aim to solve a different objective.
The objective we solve is to move towards the decision boundary, but when the recourse option causes either $x$ or the individual close to it to cross the decision boundary, we terminate the optimization one step prior to this.

\section{Computational Costs}
\label{sec:compCost}
All tests were run on a MacBook Pro, Apple M1 Pro, 16 GB. Re-running tests will take less than 1 day.

\section{User Study}
\label{sec:userStudyPDF}
Here we show our entire user study for complete transparency.
We used the German Credit dataset, but converted the currency into U.S. dollars since it was given to U.S. citizens to complete.

\begin{figure*}[!h]
  \centering
  \includegraphics[width=0.95\textwidth]{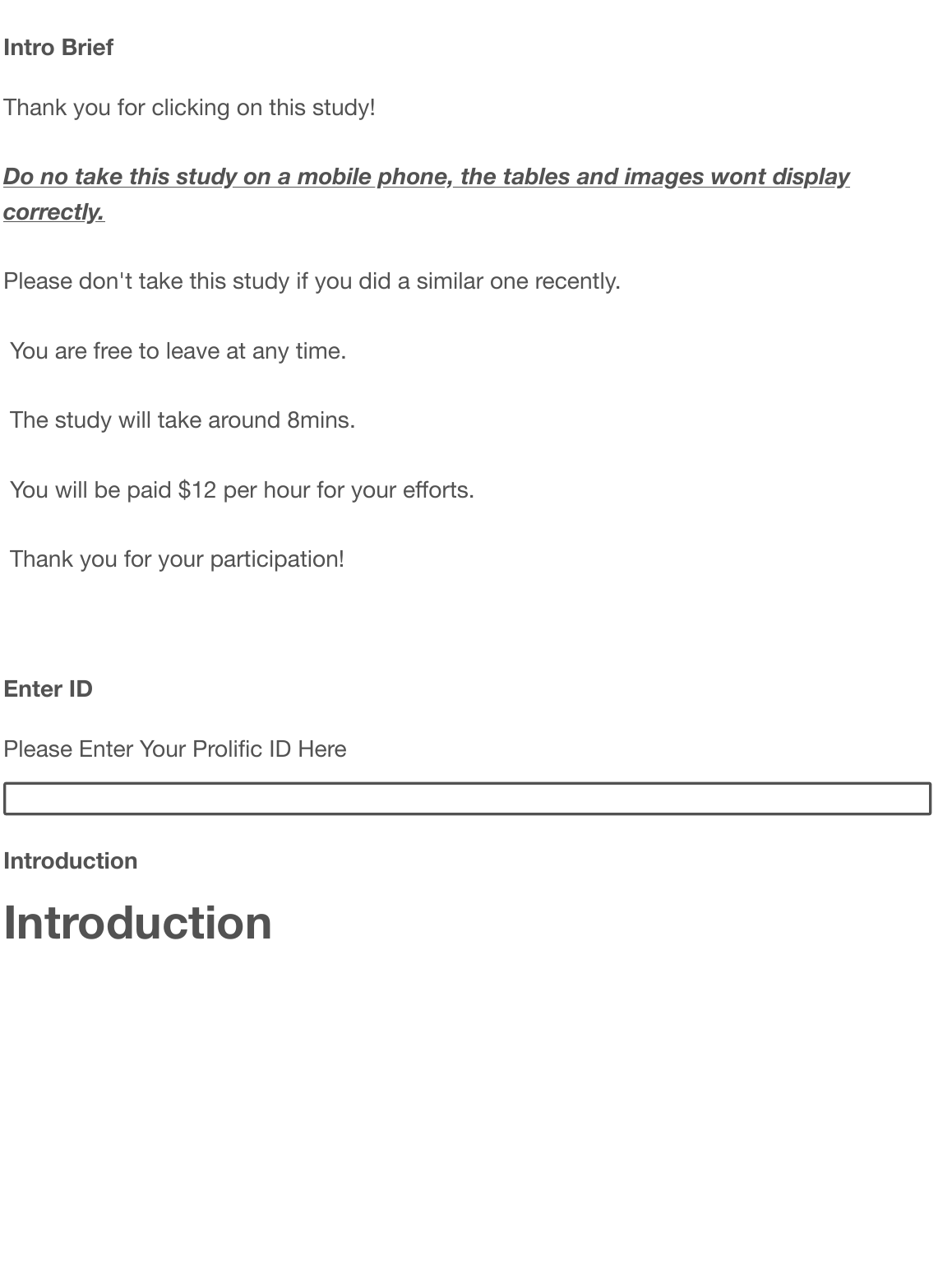}
\end{figure*}

\begin{figure*}[!h]
  \centering
  \includegraphics[width=0.95\textwidth]{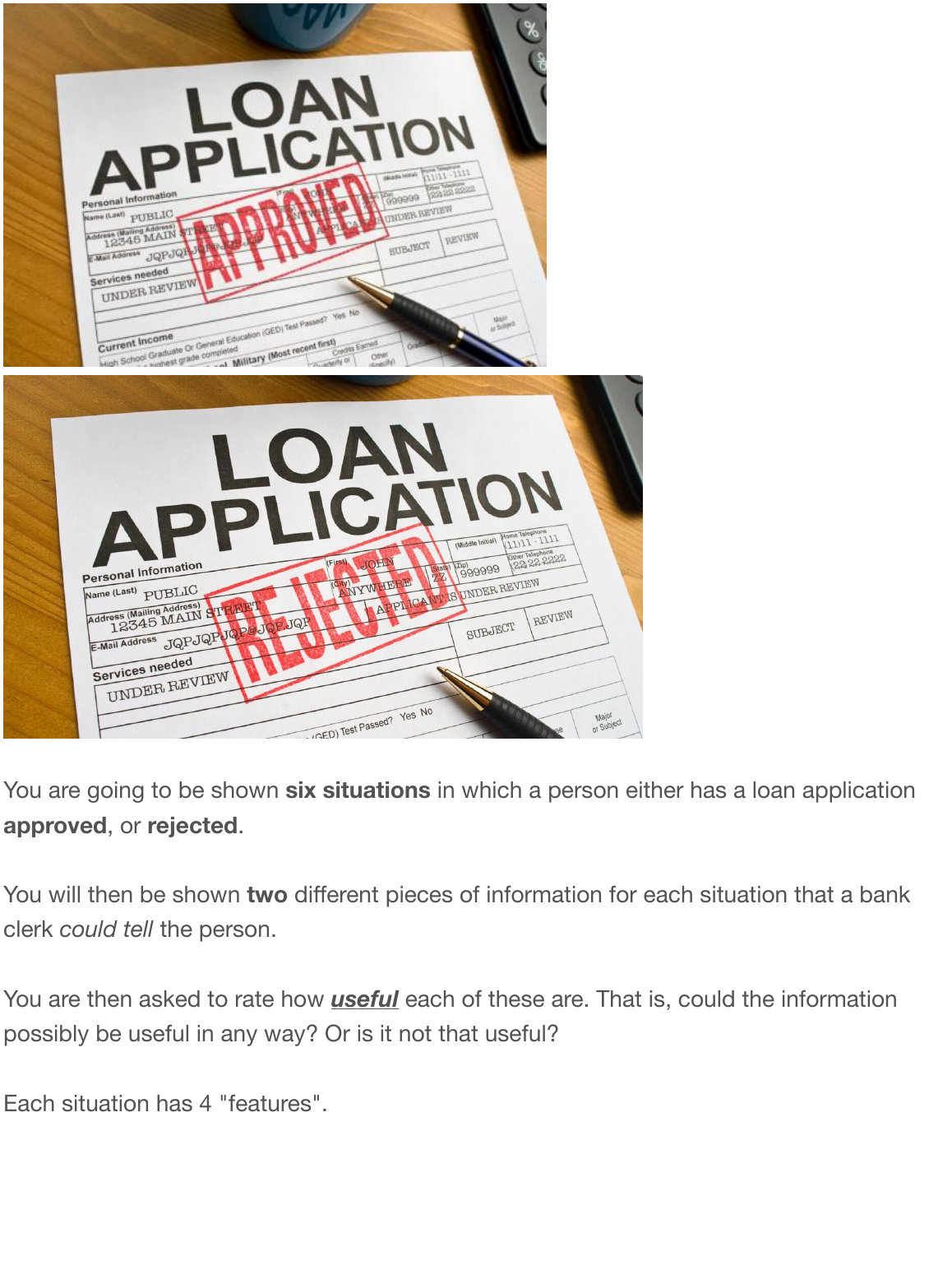}
\end{figure*}

\begin{figure*}[!h]
  \centering
  \includegraphics[width=0.95\textwidth]{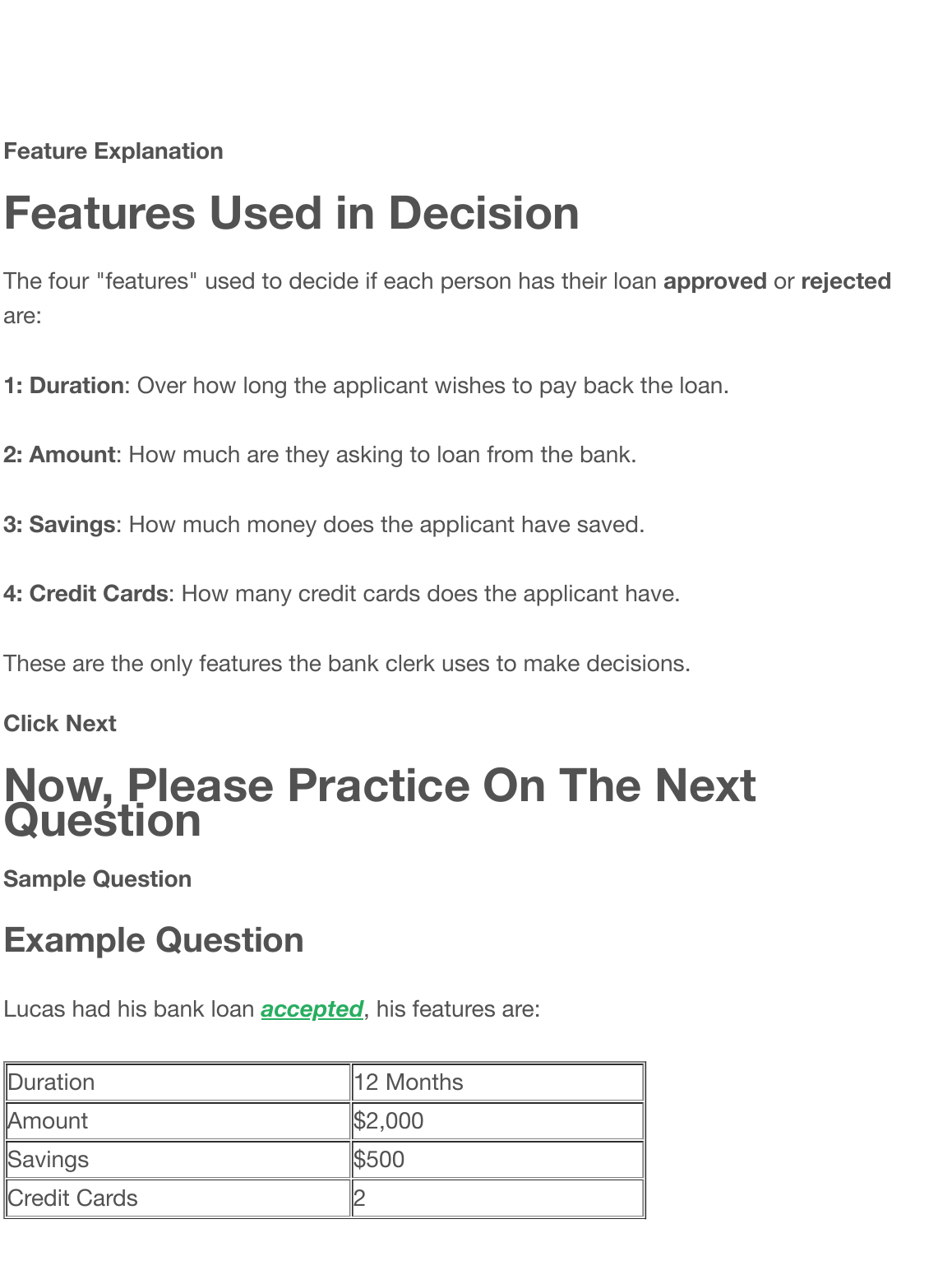}
\end{figure*}

\begin{figure*}[!h]
  \centering
  \includegraphics[width=0.95\textwidth]{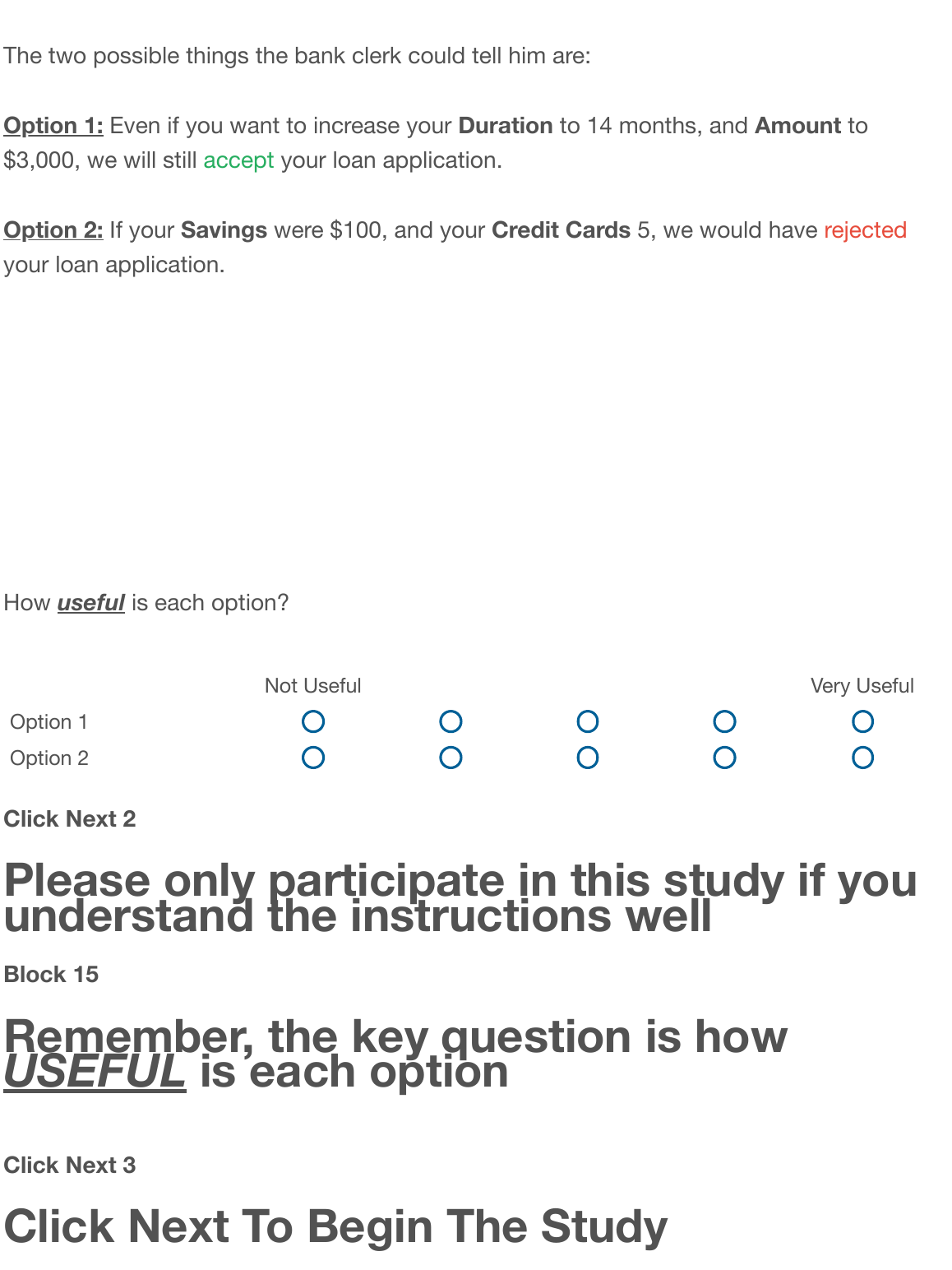}
\end{figure*}

\begin{figure*}[!h]
  \centering
  \includegraphics[width=0.95\textwidth]{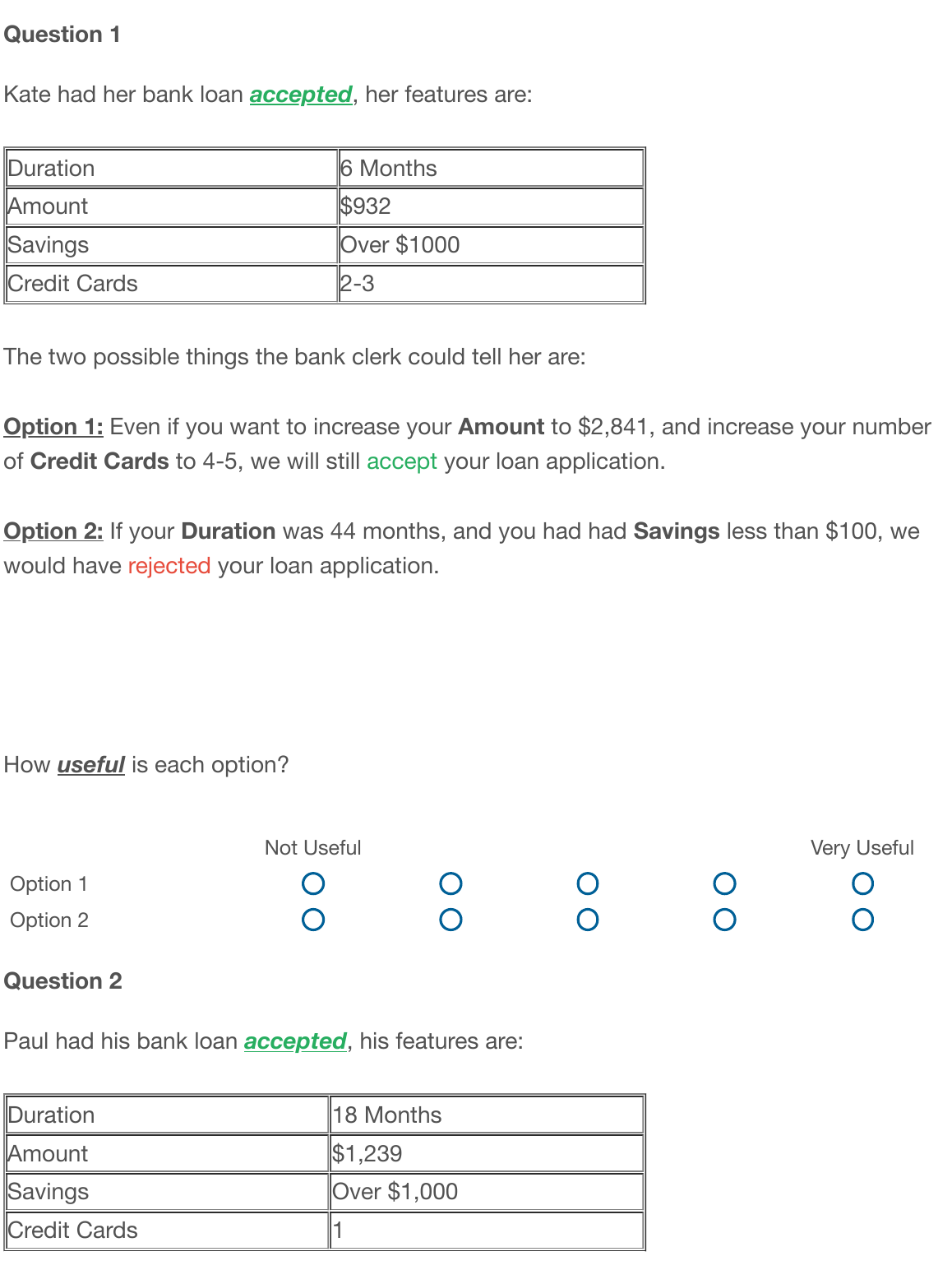}
\end{figure*}

\begin{figure*}[!h]
  \centering
  \includegraphics[width=0.95\textwidth]{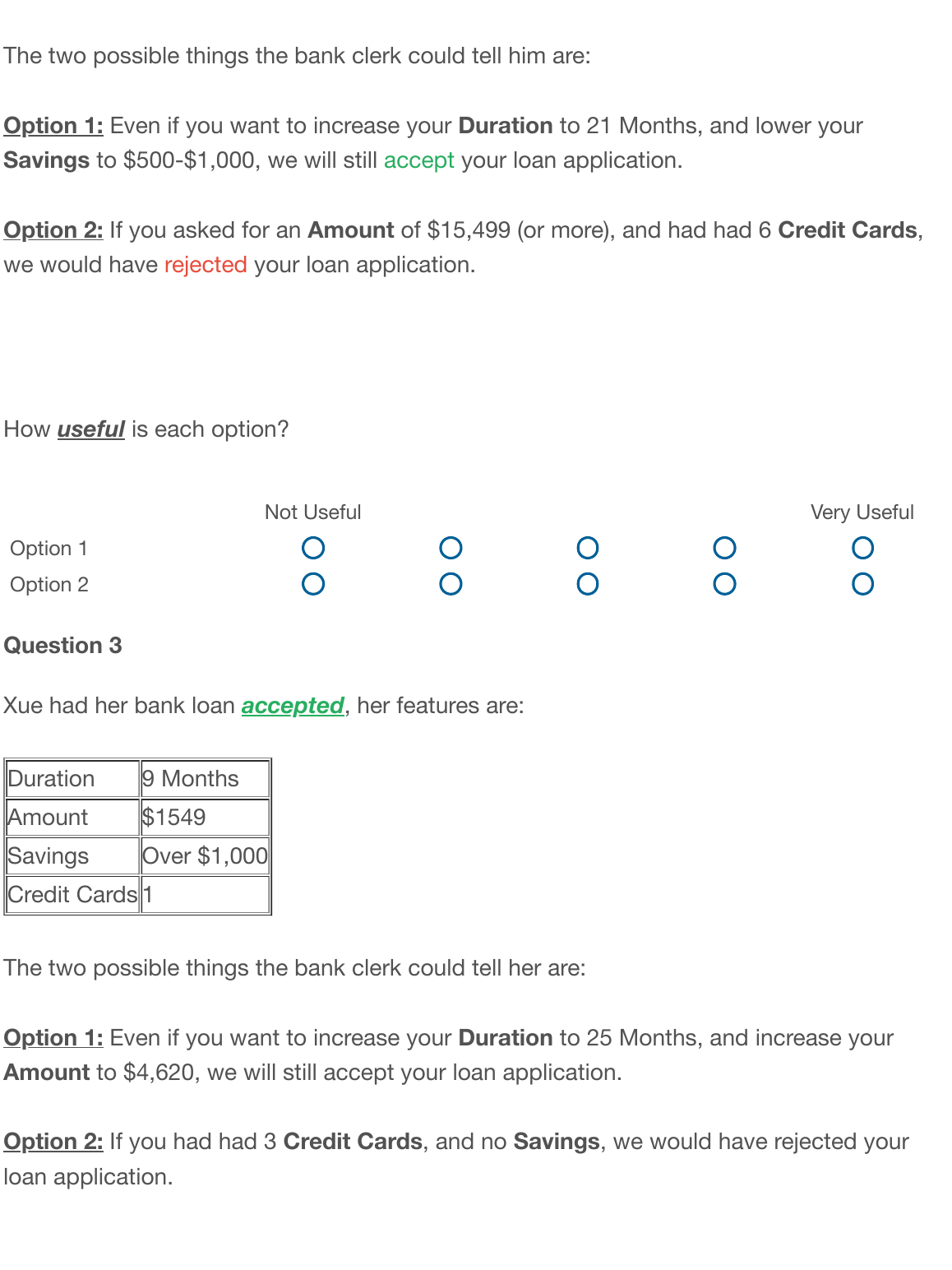}
\end{figure*}

\begin{figure*}[!h]
  \centering
  \includegraphics[width=0.95\textwidth]{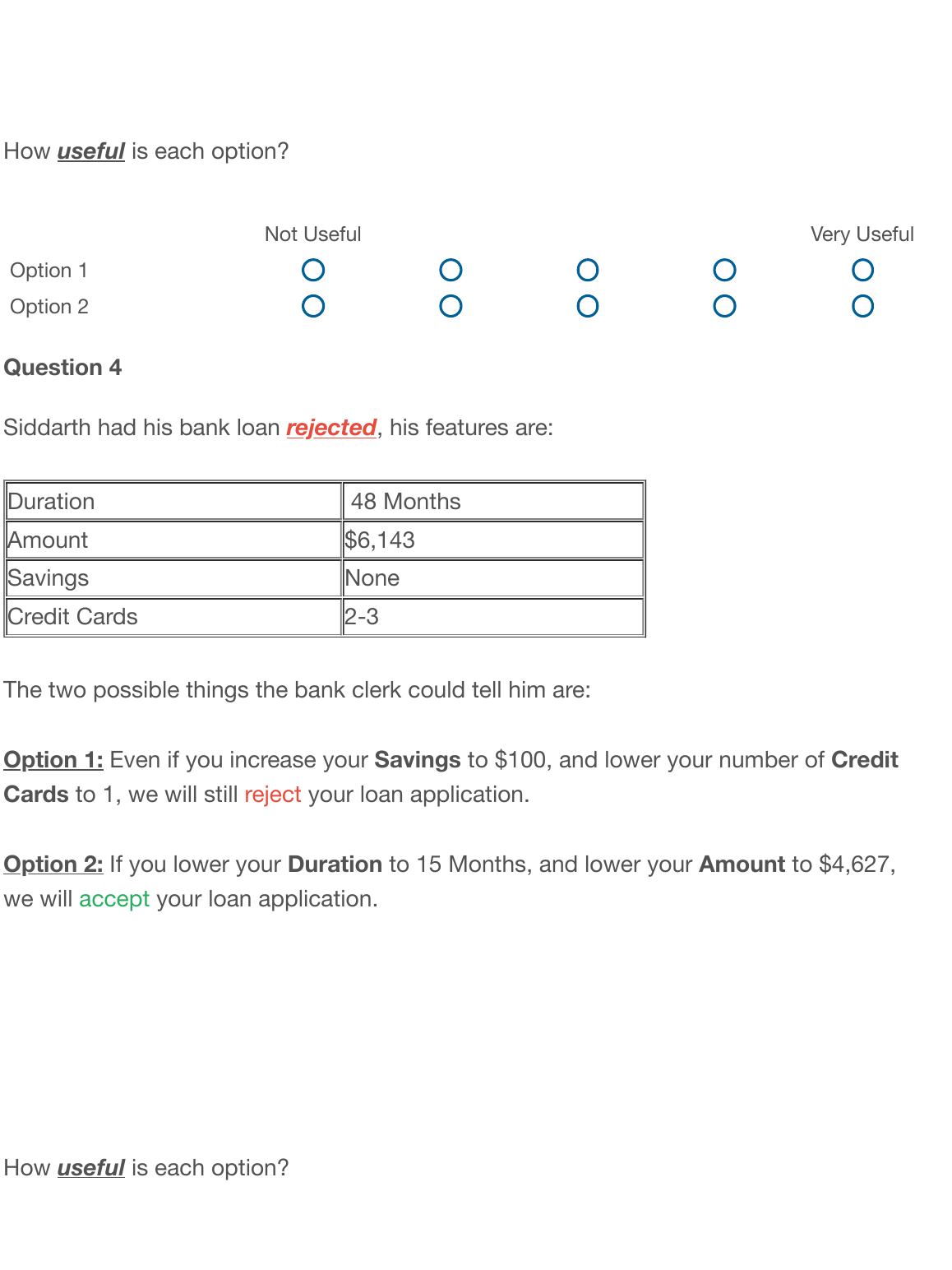}
\end{figure*}

\begin{figure*}[!h]
  \centering
  \includegraphics[width=0.95\textwidth]{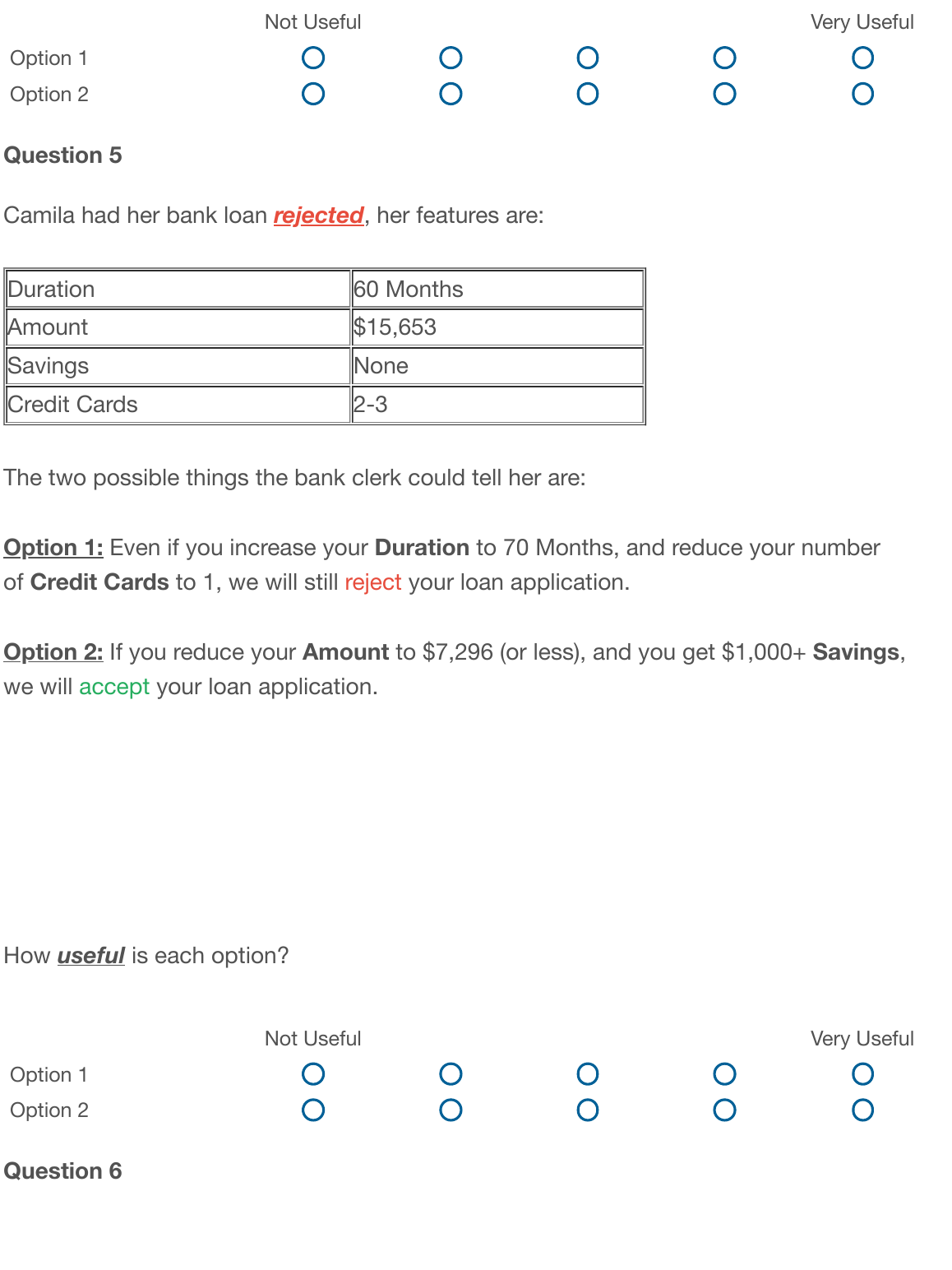}
\end{figure*}

\begin{figure*}[!h]
  \centering
  \includegraphics[width=0.95\textwidth]{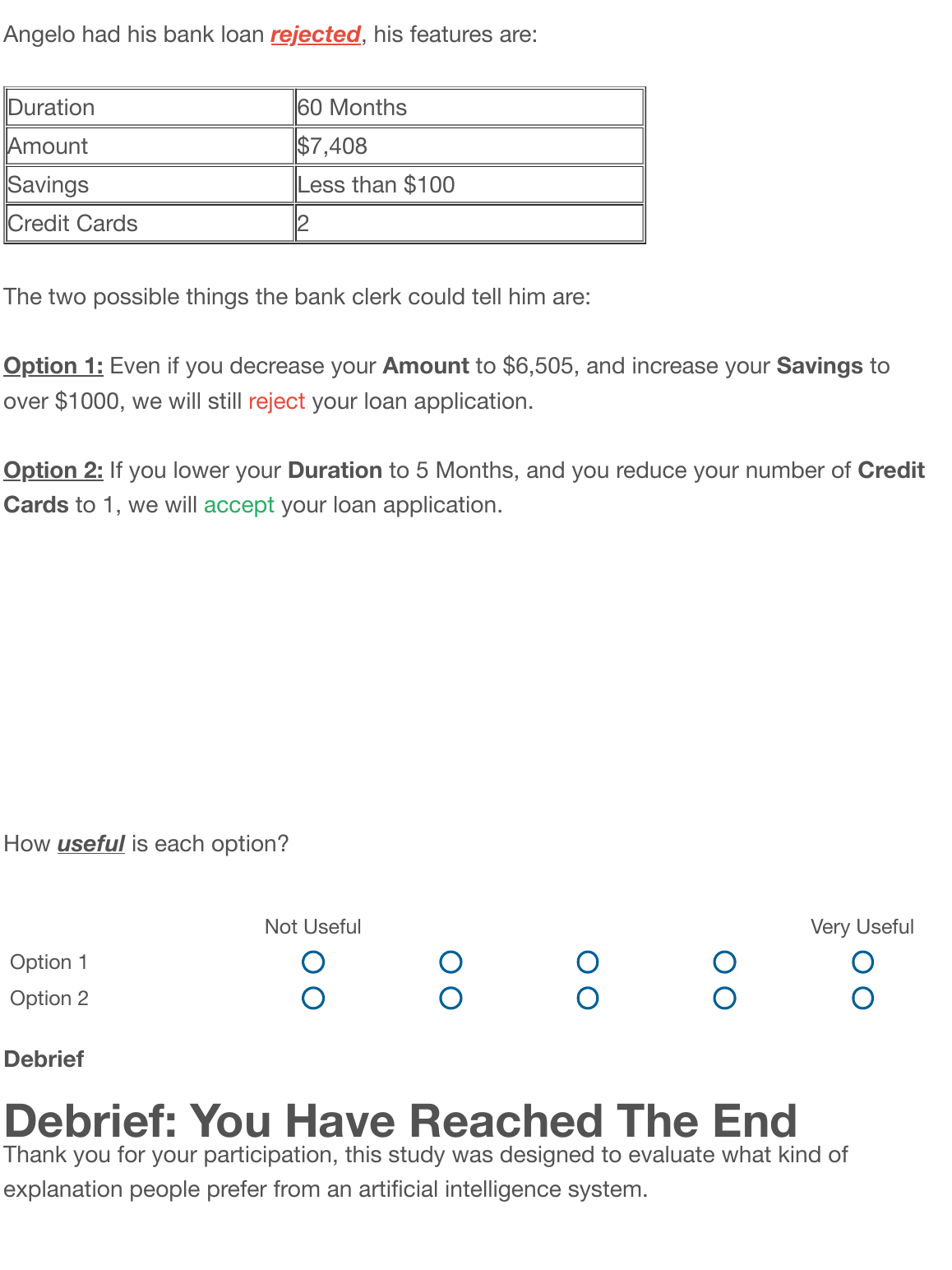}
\end{figure*}